\documentclass[12pt]{alt2020} 

\title[An adaptive stochastic optimization algorithm for resource allocation]{An adaptive stochastic optimization algorithm for resource allocation}
\usepackage{times}
\usepackage{comment}

\usepackage{hyperref}

\usepackage[utf8]{inputenc}
\usepackage[T1]{fontenc}
\usepackage{url}

\usepackage{mathtools}
\usepackage{amsmath}
\usepackage{amssymb}
\usepackage{color}
\usepackage[mathscr]{euscript}
\usepackage{dsfont}
\hypersetup{
    colorlinks,
    linkcolor={red!50!black},
    citecolor={blue!50!black},
    urlcolor={blue!80!black}
}

\usepackage[Algorithm]{algorithm}
\usepackage{algorithmic}
\usepackage{comment}

\usepackage{float}
\usepackage[justification=centering]{caption}
\usepackage{tikz, pgfplots, subcaption}

\newtheorem{assumption}{Assumption}

\newcommand{\bN}{\mathbb{N}}
\newcommand{\bP}{\mathbb{P}}
\newcommand{\bR}{\mathbb{R}}

\newcommand{\cF}{\mathcal{F}}
\newcommand{\cJ}{\mathcal{J}}

\newcommand{\sL}{\mathscr{L}}
\newcommand{\lagr}{\sL}
\newcommand{\dR}{\mathds{R}}

\newcommand{\dpart}[2]{\dfrac{\partial #1}{\partial #2}}

\newcommand{\eps}{\varepsilon}

\newcommand{\xt}{x^{(t)}}

\newcommand{\argmin}{\mbox{argmin}}
\newcommand{\argmax}{\mbox{argmax}}

\newcommand{\et}{\mbox{and}}
\newcommand{\sign}{\mbox{sign}}

\newcommand{\la}{\langle}
\newcommand{\ra}{\rangle}
\newcommand{\dd}{\, \mbox{d}}
\newcommand{\norm}[1]{\left\Vert#1\right\Vert}
\newcommand{\normi}[1]{\norm{#1}_{\infty}}
\newcommand{\normun}[1]{\norm{#1}_{1}}
\DeclarePairedDelimiter\abs{\lvert}{\rvert}
\newcommand{\st}{^{\star}}
\newcommand{\bRpp}{\bR^*_+}
\renewcommand{\t}{^{\top}}

\newcommand{\point}{\cdot\,}

\newcommand{\bigo}{\mathcal{O}}
\newcommand{\bigot}{\widetilde{\mathcal{O}}}

\newcommand{\DK}{\Delta^K}

\newcommand{\logd}{\log_2}

\newcommand{\jm}{j_{\max}}
\newcommand{\jum}{j_{1,\min}}
\newcommand{\jdm}{j_{2,\max}}

\newcommand{\rmax}{r_{\max}}
\newcommand{\smax}{s_{\max}}

\newcommand{\ie}{\textit{i.e.,}}
\newcommand{\eg}{\textit{e.g.,}}

\newcommand\numberthis{\addtocounter{equation}{1}\tag{\theequation}}

\newcommand{\sj}{\sum_{j=1}^{\jm}}
\newcommand{\sJ}{\sum_{j \in \cJ_1}}
\newcommand{\sJJ}{\sum_{j \in \cJ_2}}

\newcommand{\hgj}{\hat{g}_j}
\newcommand{\hnabla}{\widehat{\nabla}}
\newcommand{\FF}{\cF(g_1,\dots,g_{\jm})}

\newcommand{\D}[2]{\mathscr{D}^{(#1)}_{#2}}

\newcommand{\F}[2]{F^{(#1)}_{#2}}
\newcommand{\G}[2]{G^{(#1)}_{#2}}
\newcommand{\g}[2]{g^{(#1)}_{#2}}
\renewcommand{\H}[2]{H^{(#1)}_{#2}}

\newcommand{\R}[2]{R^{(#1)}_{#2}}
\newcommand{\T}[2]{T^{(#1)}_{#2}}
\newcommand{\z}[2]{z^{(#1)}_{#2}}

\newcommand{\Dij}{\D{i}{j}}

\newcommand{\Fij}{\F{i}{j}}
\newcommand{\Gij}{\G{i}{j}}
\newcommand{\gij}{\g{i}{j}}
\newcommand{\Hij}{\H{i}{j}}

\newcommand{\Rij}{\R{i}{j}}
\newcommand{\Tij}{\T{i}{j}}
\newcommand{\zij}{\z{i}{j}}

\newcommand{\LL}{\L ojasiewicz~}

\altauthor{%
 \Name{Xavier Fontaine} \Email{fontaine@cmla.ens-cachan.fr}\\
 \addr CMLA, ENS Cachan\\
       CNRS, Universit\'e Paris-Saclay
 \AND
 \Name{Shie Mannor} \Email{shie@ee.technion.ac.il}\\
 \addr Faculty of Electrical Engineering\\
       Technion, Israel Institute of Technology
 \AND
 \Name{Vianney Perchet} \Email{vianney.perchet@normalesup.org}\\
 \addr CREST, ENSAE \\ \& Criteo AI Lab, Paris%
}

\begin{document}

\maketitle


\begin{abstract}%
We consider the classical problem of sequential resource allocation where a decision maker must repeatedly divide a budget between several resources, each with diminishing returns. This can be recast as a specific stochastic optimization problem where the objective is to maximize the cumulative reward, or equivalently to minimize the regret. We construct an algorithm that is {\em adaptive} to the complexity of the problem, expressed in term of the regularity of the returns of the resources, measured by the exponent in the \L ojasiewicz inequality (or by their universal concavity parameter). Our parameter-independent algorithm recovers the optimal rates for strongly-concave functions and the classical fast rates of multi-armed bandit (for linear reward functions). Moreover, the algorithm improves existing results on stochastic optimization in this regret minimization setting for intermediate cases.
\end{abstract}

\begin{keywords}%
Stochastic optimization, online learning, adaptive algorithms, resource allocation
\end{keywords}

\section{Introduction}
In the classical resource allocation problem, a decision maker has a fixed amount of budget (money, energy, work, etc.) to divide between several resources. Each of these resources is assumed to produce a positive return for any amount of budget allocated to them, and zero return if no budget is allocated to them \citep{samuelson2005macroeconomics}. The resource allocation problem is an age-old problem that has been theoretically investigated by~\citet{koopman} and that has attracted much attention afterwards~\citep{stochastic_ra, online_ra} due to its numerous applications (\eg~production planning or portfolio selection) described for example by~\citet{gross} and~\citet{resource_allocation}. Other applications include cases of computer scheduling, where concurrent processes compete for common and shared resources. This is the exact same problem encountered in load distribution or in project management where several tasks have to be done and a fixed amount of money/time/workers has to be distributed between those tasks. Flexible Manufacturing Systems (FMS) are also an example of application domain of our problem~\citep{fms} and motivate our work. Resource allocation problems arise also in the domain of wireless communications systems, for example in the new 5G networks, due to the exponential growth of wireless data~\citep{noma}. Finally utility maximization in economics is also an important application of the resource allocation problem, which explains that this problem has been particularly studied in economics, where classical assumptions have been made for centuries~\citep{smith1776}. One of them is the \textit{diminishing returns assumption} that states that ``adding more of one factor of production, while holding all others constant, will at some point yield lower incremental per-unit returns''\footnote{See \url{https://en.wikipedia.org/wiki/Diminishing_returns}}. This natural assumption means that the reward or utility per invested unit decreases, and can be linked to submodular optimization~\citep{welfare}.

In this paper we consider the online resource allocation problem with diminishing returns. A decision maker has to partition, at each stage, \$1 between $K$ resources. Each resource has an unknown reward function which is assumed to be concave and increasing. As the problem is repeated in time,  the decision maker can gather information about the reward functions and sequentially learn the optimal allocation. We assume that the reward itself is not observed precisely, but rather a noisy version of the gradient is observed. As usually in sequential learning -- or  bandit -- problems~\citep{MAL-024}, the natural objective is to maximize the cumulative reward, or equivalently, to minimize the difference between the obtained allocation, namely the regret. 

This problem is a generalization of linear resource allocation problems,  widely studied in the last decade~\citep{NIPS2015_5931,pmlr-v83-dagan18a}, where the reward functions are assumed to be linear, instead of being  concave. Those approaches borrowed ideas from linear bandits~\citep{Dani2008StochasticLO,Abbasi-Yadkori:2011}. Several UCB-style algorithms with nearly optimal regret analysis have been proposed for the linear case. More general algorithms were also developed to optimize an unknown convex function with bandit feedback
~\citep{agarwal2011, Agrawal:2014, Agrawal:2015} to get a generic  $\bigot(\sqrt{T})$\footnote{The $\bigot(\cdot)$ notation is used to hide poly-logarithmic factors.} regret bound which is actually unavoidable with bandit feedback~\citep{Shamir13}. We consider instead that the decision maker has a noisy gradient feedback, so that the regularity of the reward mappings can be leveraged to recover faster rates (than $\sqrt{T}$) of convergence when possible.

There are several recent works dealing with (adaptive) algorithms for first order stochastic convex optimization. On the contrary to classical gradient-based methods, these algorithms are agnostic and adaptive to some complexity parameters of the problem, such as the smoothness or strong convexity parameters. For example,~\citet{judistky} proposed an adaptive algorithm to optimize uniformly convex functions and~\citet{RS15} generalized it with an epoch-based Gradient Descent algorithm using active learning techniques, also to minimize uniformly convex functions. Both obtain optimal bounds in $\bigot\left(T^{-\rho/(2 \rho-2)} \right)$ for the function-error $\norm{f(x_t)-f^*}$ where $f$ is supposed to be $\rho$-uniformly convex (see Subsection~\ref{ssec:complexity} for a reminder on this regularity concept).
However those algorithms would only achieve a $\sqrt{T}$ regret (or even a linear regret) because they rely on a structure of phases of unnecessary lengths.  So in that setting, regret minimization appears to be  much more challenging  than function-error minimization. To be precise, we  actually consider an even weaker concept of regularity than uniform convexity:  the \LL inequality~\citep{PMIHES_1988__67__5_0,Bolte}. Our objective is to devise an algorithm that can leverage this assumption, without the prior knowledge of the \LL exponent, \ie~to construct an adaptive algorithm unlike precedent approaches~\citep{Kar16}.

\medskip
\noindent \textbf{High-level description of the algorithms and organization of the paper.}
The algorithm we are going to introduce is based on the concept of dichotomy, or  binary search, which has already been slightly investigated in stochastic optimization~\citep{BZ74, minimax, RS13}. The specific case of $K=2$ resources is studied in Section~\ref{sec:k2}. The algorithm proposed is quite simple: it queries a point repeatedly, until it learns the sign of the gradient of the reward function, or at least with arbitrarily high probability. Then it proceeds to the next step of a standard binary search.

We will then consider, in Section~\ref{sec:k4}, the case of $K\geq3$ resources by defining a binary tree of the $K$ resources and handling each decision using the $K=2$ algorithm as a black-box. Our main result can be stated as follows: if the base reward mappings of the resources are $\beta$-\LL functions, then our algorithm has a  $\bigot(T^{-\beta/2})$ regret bound if $\beta \leq 2$ and $\bigot(T^{-1})$ otherwise. We notice that for $\beta \leq 2$ we recover existing bounds (but for the more demanding regret instead of function-error minimization)~\citep{judistky, RS15} since a $\rho$-uniformly convex function can be proven to be $\beta$-\LL with $\beta=\rho/(\rho-1)$. We complement our results with a lower bound that indicates the tightness of these bounds. Finally we corroborate  our theoretical findings with some experimental results, postponed to Appendix \ref{Appendix:Expe}.

Our main contributions are the design of an efficient algorithm to solve the resource allocation problem with concave reward functions. We show that our algorithm is adaptive to the unknown complexity parameters of the reward functions. Moreover we propose a unified analysis of this algorithm for a large class of functions. It is interesting to notice that our algorithm can be seen as a first-order convex minimization algorithm for separable loss functions. The setting of separable loss functions is still common in practice, though not completely general. Furthermore we prove that our algorithm outperforms other convex minimization algorithms for a broad class of functions. Finally we exhibit links with bandit optimization and we recover classical bandit bounds within our framework, highlighting the connection between bandits theory and convex optimization.

First, let us  introduce in Section \ref{sec:Model} the following general model and the different regularity assumptions mentioned above.

\section{Model and Assumptions}
\label{sec:Model}

\subsection{Problem Setting}
Assume a decision maker has access to $K \in \bN^*$ different resources. We assume naturally that the number of resources $K$ is not too large (or infinite). At each time step $t \in \bN^*$, the agent has to split a total budget of weight $1$ and to allocate $x_k^{(t)}$ to each resource $k \in [K]$ which generates the reward $f_k(x_k(t))$. Overall, at this stage, the reward of the decision maker is then
\[
F(x^{(t)}) = \sum_{k \in [K]} f_k(x_k^{(t)}) \quad \text{ with } \ x^{(t)} = (x_1^{(t)}, \ldots, x_K^{(t)}) \in \DK,
\]
where the simplex $\DK = \big\{ (p_1,\ldots,p_K) \in \bR^K_+; \,  \sum_k p_k =1\big\}$ is the set of possible convex weights.

We note $x^\star \in \DK$ the optimal allocation that maximizes $F$ over $\DK$; the objective of the decision maker is to maximize the cumulated reward, or equivalently to  minimize the regret $R(T)$, defined as the difference between the optimal reward $F(x^\star)$ and the average reward over $T \in \bN^*$ stages:
\[
R(T) = F(x^\star) - \frac{1}{T} \sum_{t=1}^T  \sum_{k=1}^K f_k(x_k^{(t)}) = \max_{x \in \DK} F(x)  - \frac{1}{T} \sum_{t=1}^T F(x^{(t)}).
\]
The following diminishing return assumption on the reward functions $f_k$ is natural and  ensures that $F$ is concave and continuous, ensuring the existence of $x^\star$.

\begin{assumption}
The reward functions $f_k : [0,1] \to \bR$ are concave, non-decreasing and $f_k(0)=0$. Moreover we assume that they are differentiable, $L$-Lipschitz continuous and $L'$-smooth.
\end{assumption}

This assumption means that the more the decision maker invest in a resource, the greater the revenue. Moreover,  investing $0$ gives nothing in return. Finally  the marginal increase of revenue decreases.

We now describe the feedback model. At each time step the decision maker observes a noisy version of $\nabla F(x^{(t)})$, which is equivalent here to observing each $\nabla f_k(x_k^{(t)}) + \zeta_k^{(t)}$, where $\zeta_k^{(t)} \in \bR$ is some white bounded noise. 
The assumption of noisy gradients is classical in stochastic optimization and is similarly relevant for our problem: this assumption is quite natural as the decision maker can evaluate, locally and with some noise, how much a small increase/decrease of an allocation $x_k^{(t)}$ affects the reward. 

Consequently, the decision maker faces the problem of stochastic optimization of a concave and separable function over the simplex (yet with a cumulative regret minimization objective). Classical stochastic gradient methods from stochastic convex optimization would guarantee that the average regret decreases as $\bigot\left((K/T)^{1/2}\right)$ in general and as $\bigot\left(K/T\right)$ if the $f_k$ are known to be strongly concave. However, even without strong concavity, we claim that it is possible to obtain better regret bounds than $\bigot\left((K/T)^{1/2}\right)$ and, more importantly, to be adaptive to some complexity parameters.

The overarching objective is then to leverage the specific structure of this natural problem to provide a generic algorithm that is naturally \textbf{adaptive} to some complexity measure of the problem. It will, for instance, \textbf{interpolate} between the non-strongly concave and the strongly-concave rates without depending on the strong-concavity parameter, and \textbf{recover} the fast rate of classical multi-armed bandit (corresponding more or less to the case where the $f_k$ functions are linear). Existing algorithms for adaptive stochastic convex optimization~\citep{RS15, judistky} are not applicable in our case since they work for function-error minimization and not regret minimization (because of the prohibitively large stage lengths they are using).

\subsection{The complexity class}
\label{ssec:complexity}

As mentioned before, our algorithm will be adaptive to some general complexity parameter of the set of functions ${\mathcal{F}}=\{f_1,\ldots,f_K\}$, which relies on the \LL inequality~\citep{PMIHES_1988__67__5_0,Bolte} that we state now, for concave functions (rather than convex).

\begin{definition}
A function $f: \bR^d \to \bR$ satisfies the \LL inequality with respect to $\beta \in [1,+\infty)$ on its domain $\mathcal{X} \subset \bR^d$ if there exists a constant $c >0$ such that
\[
\forall x \in \mathcal{X}, \ \max_{x^* \in \mathcal{X}} f(x^*)   - f(x)  \leq c \| \nabla f(x)\|^\beta .
\]
\end{definition}

Given two functions $f ,g : [0,1] \to \bR$, we say that they satisfy \textbf{pair-wisely} the \LL inequality with respect to $\beta \in [1,+\infty)$ if the function $\left(z \mapsto f(z) + g(x-z) \right)$ satisfies the \LL inequality on $[0,x]$ with respect to $\beta$ for every $x \in [0,1]$. 

It remains to define the finest class of complexity of a set of functions $\mathcal{F}$. It is defined with respect to binary trees, whose nodes and leaves are labeled by functions. The trees we consider are constructed as follows. Starting from a finite binary tree of depth $\lceil\log_2(|\mathcal{F}|)\rceil$, its leaves are labeled with the different functions in $\mathcal{F}$ (and 0 for the remaining leaves if $|\mathcal{F}|$ is not a power of 2). The  parent node of $f_{\mathrm{left}}$ and $f_{\mathrm{right}}$  is then labeled by  the function $x \mapsto \max_{z \leq x} f_{\mathrm{left}}(z) + f_{\mathrm{right}}(x-z)$. 

We say now that $\mathcal{F}$ satisfies \textbf{inductively} the \LL inequality for $\beta \geq 1$ if in any binary tree labeled as above, any two siblings\footnote{To be precise, we could only require that this property holds for any siblings that are not children of the root. For those two, we only need that the mapping $f_{\mathrm{left}}(z) + f_{\mathrm{right}}(1-z)$ satisfies the local \LL inequality.} satisfy pair-wisely the \LL inequality for $\beta$.

Since the previous definition is quite intricate we can focus on some easier insightful sub-cases:

\begin{description}
\item[Uniformly concave functions~\citep{judistky}] A function $f:\bR^d\to \bR$ is uniformly-concave with parameters $\rho\geq 2$ and $\mu>0$ if and only if for all $x,y \in \bR^d$ and for all $\alpha \in [0,1]$,
\[
f(\alpha x + (1-\alpha)y) \geq \alpha f(x)+(1-\alpha)f(y)+\dfrac{\mu}{2}\alpha(1-\alpha)\left[\alpha^{\rho-1}+(1-\alpha)^{\rho-1}\right]\norm{x-y}^{\rho}.
\]
 If all functions $f_k$ are $(\rho_k,\mu_k)$-uniformly convex, then the relevant complexity parameter (for the rate of convergence) is $\beta_{\mathcal{F}}=\frac{\rho_{\mathcal{F}}}{\rho_{\mathcal{F}}-1}$ where $\rho_{\mathcal{F}}:=\max_k \rho_k$.
\item[Tsybakov Noise Condition (TNC)~\citep{RS15}] A function $f:\bR^d\to \bR$  satisfies the global TNC if with parameters $\kappa \geq 2$ and $\mu>0$ if and only if for all $x,y \in \bR^d$,
\[|f(x)-f(y)| \geq  \mu \| x-y\|^\kappa. \]
 If all functions $f_k$ satisfies $(\kappa_k,\mu_k)$-TNC, then the relevant complexity parameter (for the rate of convergence) is $\beta_{\mathcal{F}}=\frac{\kappa_{\mathcal{F}}}{\kappa_{\mathcal{F}}-1}$ where $\kappa_{\mathcal{F}}:=\max_k \kappa_k$.
\end{description}
 
More details about the \LL inequality (as well as examples and counter-examples) and its links with uniform convexity can be found in Appendix~\ref{app:loj}. Additional examples of class of functions satisfying inductively the \LL inequality can be found in Appendix~\ref{SE:complex}.

One could ask why the class of \LL functions is interesting. A result of~\citet{loja} shows that all analytic functions satisfy the (local) \LL inequality with a parameter $\beta > 1$. This is a strong result motivating our interest for the class of functions satisfying the \LL inequality. More precisely we prove the following proposition in Appendix~\ref{SE:complex}.

\begin{proposition}
\label{prop:analytic}
If the functions $\{f_1, \dots, f_K\}$ are real analytic and strictly concave then the class $\mathcal{F}$ satisfy inductively the \LL inequality with a parameter $\beta_{\mathcal{F}}>1$.
\end{proposition}

In the following section, we introduce a \textbf{generic, parameter free algorithm} that is \textbf{adaptive to the complexity $\beta_\mathcal{F}\in [1,+\infty)$ of the problem}. Note that $\beta_{\cF}$ is not necessarily known by the agent and therefore the fact that the algorithm is adaptive to the parameter is particularly interesting.
The simplest case $K=2$ provides many insights and will be used as a sub-routine for more resources. Therefore, we will first focus on this case.

\section{Stochastic Gradient Feedback for $K=2$}
\label{sec:k2}
We first focus on only $K=2$ resources. In this case, we  rewrite the reward function $F$ as
\[F(x)=f_1(x_1)+f_2(x_2)=f_1(x_1)+f_2(1-x_1).\] For the sake of clarity we simply note $x=x_1$ and we define $g(x)\triangleq F(x)-F(x\st)$. Note that $g(x\st)=0$ and that $g$ is a non-positive concave function. Using these notations, at each time step $t$ the agent chooses $\xt \in [0,1]$, suffers $\abs{g(\xt)} $ and observes $g'(\xt)+\eps_t$ where $\eps_t \in [-1,1]$ i.i.d.

\subsection{Description of the main algorithm}

The basic  algorithm we follow to optimize $g$ is a binary search. Each query point $x$ (for example $x=1/2$) is sampled repeatedly and sufficiently enough (as long as 0 belongs to some confidence interval) to guarantee that the sign of $g'(x)$ is known with arbitrarily high probability, at least $1-\delta$.

\begin{algorithm}[h]
\caption{Binary search algorithm \label{algo}}
\algsetup{indent=2em}
\begin{algorithmic}[1]
\REQUIRE $T$ time horizon, $\delta$ confidence parameter
\STATE Search interval $I_0\leftarrow [0,1]$ ;  $t \leftarrow 1$ ; $j \leftarrow 1$
\WHILE{$t \leq T$}
\STATE $x_j \leftarrow \textrm{center}(I_{j-1})$; $S_j \leftarrow 0$; $N_j \leftarrow 0$
\WHILE{$0 \in \Big[\frac{S_j}{N_j}  \pm \sqrt{\frac{2\log(\frac{2T}{\delta})}{N_j}}\Big]
$}
\STATE Sample $x_j$ and get $X_t$, noisy value of $\nabla g(x_j)$
\STATE $S \leftarrow S_j+ X_t$, $N_j \leftarrow N_j+1$
\ENDWHILE
\IF{$\frac{S_j}{N_j}  > \sqrt{\frac{2\log(\frac{2T}{\delta})}{N_j}}$}
\STATE $I_j\leftarrow [x_j,\max(I_{j-1})]$
\ELSE
\STATE $I_j\leftarrow [\min(I_{j-1}),x_j]$
\ENDIF
\STATE $t \leftarrow t+N_j$ ;  $j \leftarrow j+1$
\ENDWHILE
\RETURN $x_j$
\end{algorithmic}
\label{Algo:K2}
\end{algorithm}

Algorithm~\ref{Algo:K2}  is not conceptually difficult (but its detailed analysis of performances is however): it is just a binary search where each query point is sampled enough time to be sure on which ``direction'' the search should proceed next. Indeed, because of the concavity and monotone assumptions on $f_1$ and $f_2$, if $x < x\st$ then 
\[
x < x\st \Longleftrightarrow \nabla g(x) = \nabla f_1(x) - \nabla f_2(1-x) <0 \ .
\]
By getting enough noisy samples of $\nabla g(x)$, it is possible to decide, based on its sign, whether $x\st$ lies on the right or the left of $x$. If $x_j$ is the $j$-th point queried by the binary search (and letting $\jm$ be the total number of different queries), we  get  that the binary search is successful with high probability, i.e.,  that with probability at least $1-\delta T$ for each $j \in \{1,\ldots,\jm\}$, $\abs*{x_j-x\st}\leq 2^{-j}$. We also call $N_j$ the actual number of samples of $x_j$ which is bounded by $8\log(2T/\delta)/\abs{g'(x_j)^2}$ by Lemma~\ref{lemma:precision}, whose proof can be found in Appendix~\ref{Appendix:Tech}.
\begin{lemma}
\label{lemma:precision}
Let $x \in [-1,1]$ and $\delta \in (0,1)$. For any random variable $X \in [x-1,x+1]$ of expectation $x$, at most $N_x=\dfrac{8}{x^2}\log\left(2T/\delta \right)$ i.i.d.\ samples $X_1,X_2,\ldots,X_n$ are needed to figure out the sign of $x$ with probability at least $1-\delta $. Indeed, one just need stop sampling as soon as 
\[
0 \not\in \left[\frac{1}{n} \sum_{t=1}^n X_t \pm \sqrt{\frac{2\log(2T/\delta)}{n}}\right]
\]
and determine the sign of $x$ is positive if $\frac{1}{n} \sum_{t=1}^n X_t \geq \sqrt{\frac{2\log(2T/\delta)}{n}}$ and negative otherwise.
\end{lemma}

The regret of the algorithm then rewrites  as
\[
\label{eq:master}
R(T)=\dfrac{1}{T}\sum_{t=1}^T \abs{g(\xt)}=\dfrac{1}{T}\sum_{j=1}^{\jm} N_j \abs{g(x_j)}\leq\dfrac{8}{T}\log(2T/\delta)\sum_{j=1}^{\jm}\dfrac{\abs{g(x_j)}}{g'(x_j)^2}.
\numberthis
\]
Our analysis of the algorithm performances are based on the control of the last sum  in Equation~\eqref{eq:master}.
\subsection{Strongly concave functions}

First, we consider the case  where the functions $f_1$ and $f_2$ are strongly concave. 

\begin{theorem}
\label{TH:K2Strongly}
If the algorithm is run with $\delta=2/T^2$ and if $g$ is a $L'$-smooth and $\alpha$-strongly concave function on $[0,1]$, then there exists a universal positive constant $\kappa$ such that
\[
\mathds{E} R(T) \leq  \dfrac{\kappa}{\alpha}\dfrac{\log(T)}{T}.
\]
\end{theorem}

This results shows that our algorithm reaches the same rates as the stochastic gradient descent in the smooth and strongly concave case. The proof is delayed to Appendix \ref{Appendix:K2} for the sake of fluency.

\subsection{Analysis in the non-strongly concave case}

We now consider the case where  $g$ is only concave, without being necessarily strongly concave.  
\begin{theorem}\label{TH:K2Beta}
Assume that $g$ satisfies the local \LL inequality w.r.t.\ $\beta \geq 1$ and $c>0$ and that the algorithm  is run with $\delta= 2/T^2$. Then there exists a universal constant $\kappa>0$ such that
\begin{description}
\item[\hspace{0.5cm} in the case where $\beta >2$,]  $\mathds{E}[R(T)]\leq \kappa \dfrac{c^{2/\beta} L^{1-2/\beta}}{1-2^{2/\beta-1}}\dfrac{\log(T)}{T}$;
\item[\hspace{0.5cm} in the case where $\beta\leq2$,]   $\mathds{E}[R(T)]\leq \kappa\cdot c \left(\dfrac{\log(T)^2}{T}\right)^{\beta/2}$.
\end{description}
\end{theorem}

The proof of Theorem~\ref{TH:K2Beta} relies on bouding the sum in Equation~\eqref{eq:master}, which can be recast as a constrained minimization problem. It is postponed to Appendix~\ref{Appendix:K2} for clarity reasons.

\subsection{Lower bounds}

We now provide a lower bound for our problem that indicates that our rates of convergence are optimal up to $\mathrm{poly}(\log(T))$ terms.  For $\beta \geq 2$, it is trivial to see that no algorithm can have a regret smaller than $\Omega(1/T)$, hence we shall focus on $\beta \in [1,2]$.

\begin{theorem}\label{TH:lower} Given the horizon $T$ fixed, for any algorithm, there exists a pair of functions  $f_1$ and $f_2$ that are concave, non-decreasing and such that $f_i(0)=0$, such that 
\[
\mathds{E} R(T) \geq c_\beta T^{-\frac{\beta}{2}}
\]
where $c_\beta>0$ is some constant independent of $T$.
\end{theorem}

The proof  and arguments are rather classical now \citep{Shamir13,BacPer16}: we exhibit two pairs of functions whose gradients are $1/\sqrt{T}$-close with respect to the uniform norm. As no algorithm can distinguish between them with arbitrarily high probability, the regret will scale more or less as the difference between those functions which is as expected of the order of $T^{-\beta/2}$. More details can be found in Appendix~\ref{Appendix:K2}.

\subsection{The specific case of linear (or dominating) resources - the Multi-Armed Bandit case}

We focus in this section on the specific case where the resources have linear efficiency, meaning that $f_i(x) = \alpha_i x$ for some unknown parameter $\alpha_i \geq 0$. In that case, the optimal allocation of resource consists in putting all the weights to the resource with the highest parameter $\alpha_i$.

More generally, if $f_1'(1) \geq f_2'(0)$, then one can easily check that the optimal allocation consists in putting again all the weight to the first resource (and, actually, the converse statement is also true).

It happens that in this specific case, the learning is fast as it can be seen as a particular instance of Theorem~\ref{TH:K2Beta} in the case where $\beta>2$. Indeed, let us assume that $\argmax_{x \in \bR} g(x)>1$, meaning that $\max_{x\in[0,1]} g(x)=g(1)$, so that, by concavity of $g$ it holds that $g'(x)\geq g'(1)>0$ thus $g$ is increasing on $[0,1]$. 
In particular, this implies that for every $\beta>2$:
\[
\forall x \in [0,1], \,g(1)-g(x)=\abs{g(x)} \leq g(0) \leq \dfrac{g(0)}{g'(1)^{\beta}}g'(1)^{\beta} \leq \dfrac{g(0)}{g'(1)^{\beta}} g'(x)^{\beta}=c\abs{g'(x)}^{\beta},
\]
showing that $g$ verifies the \LL inequality for every $\beta>2$ and with constant $c=g(0)/g'(1)^{\beta}$. As a consequence, Theorem~\ref{TH:K2Beta} applies and we obtain fast rates of convergence in $\bigo\left(\log(T)/T\right)$. 

However, we propose in the following an alternative analysis of the algorithm for that specific case. Recall that regret can be bounded as
\[
R(T)=\dfrac{8}{T}\log(2T/\delta)\sum_{j=1}^{j_{\max}} \dfrac{\abs{g(x_j)}}{g'(x_j)^2} =\dfrac{8}{T}\log(2T/\delta) \sum_{j=1}^{j_{\max}} \dfrac{\abs{g(1-1/2^j)}}{g'(1-1/2^j)^2}.
\]

We now notice that
\[\Big|g\left(1-2^{-j}\right)\Big|=g(1)-g\left(1-2^{-j}\right)=\int_{1-1/2^j}^1 g'(x)\dd x \leq 2^{-j} g'\left(1-2^{-j}\right).\]

And finally we obtain the following bound on the regret:
\[
R(T)\leq \dfrac{8}{T}\log(2T/\delta)\sum_{j=1}^{j_{\max}} \frac{1}{2^j}\frac{1}{g'(1)} \leq \dfrac{8}{T}\dfrac{\log(2T/\delta)}{g'(1)}\leq\dfrac{24}{\Delta}\dfrac{\log(T)}{T}\]
since $g'(1-1/2^j) > g'(1)$ and with the choice of $\delta=2/T^2$. We have noted $\Delta \triangleq g'(1)$ in order to enlighten the similarity with the multi-armed bandit problems with $2$ arms. We have indeed $g'(1)=f_1'(1)-f_2'(0)>0$ which can be seen as the gap between both arms. It is  especially true in the linear case where $f_i(x) =\alpha_i x$ as $\Delta=|\alpha_1-\alpha_2|$ and the gap between arms is by definition of the multi-armed bandit problem $|f(1)-f(0)| = |\alpha_1-\alpha_2|$.


\section{Stochastic gradient feedback and $K\geq 3$ resources}
\label{sec:k4}
We now consider the case with more than $2$ resources. The generic algorithm still relies on binary searches as in the previous section with $K=2$ resources, but we have to imbricate them in a tree-like structure to be able to leverage the \LL inequality assumption. The goal of this section is to present our algorithm and to prove the following theorem, which is a generalization of Theorem~\ref{TH:K2Beta}.

\begin{theorem}\label{TH:K4Beta}
Assume that $\mathcal{F}=\{f_1,f_2,\ldots,f_k\}$ satisfies inductively the \LL inequality w.r.t.\ the parameters $\beta_\mathcal{F} \geq 1$ and $c>0$.  Then there exists a universal constant $\kappa>0$ such that our algorithm, run with with  $\delta=2/T^2$, ensures
\begin{description}
\item[\hspace{0.5cm} in the case $\beta_\mathcal{F} >2$,] then  $\mathds{E}[R(T)]\leq\kappa\dfrac{c^{2/\beta_\mathcal{F}}L^{1-2/\beta_\mathcal{F}}}{1-2^{2/\beta_\mathcal{F}-1}} K \dfrac{\log(T)^{\log_2(K)}}{T}$;
\item[\hspace{0.5cm}in the case $\beta_\mathcal{F}\leq2$,] then  $\mathds{E}[R(T)] \leq \kappa \cdot c K \left(\dfrac{\log(T)^{\log_2(K)+1}}{T}\right)^{\beta_\mathcal{F}/2}$.
\end{description}
\end{theorem}

Let us first mention why the following natural extension of the algorithm for $K=2$ does not work. Assume that the algorithm would sample repeatedly a point $x \in \Delta^K$ until the different confidence intervals around the gradient $\nabla f_k(x_k)$ do not overlap. When this happens with only $2$ resources, then it is known that the optimal $x\st$ allocates more weight to the resource with the highest gradient and less weight to the resource with the lowest gradient. This property only holds partially for $K\geq 3$ resources. Given $x \in \DK$, even if we have a (perfect) ranking of gradient $\nabla f_1(x_1) > \ldots > \nabla f_K(x_K)$ we can only infer that $x^\star_1 \geq x_1$ and $x^\star_K \leq x_K$. For intermediate gradients we cannot (without additional assumptions) infer the relative position of $x^\star_j$ and $x_j$. 

To circumvent this issue, we are going to build a  binary tree, whose leaves are labeled arbitrarily from $\{f_1,\ldots,f_K\}$ and we are going to run inductively the algorithm for $K=2$ resources at each node, \ie~between its children $f_{\mathrm{left}}$ and $f_{\mathrm{right}}$. The main difficulty is that we no longer have unbiased samples of the gradients of those functions (but only those located at the leaves).

\subsection{Insights on the main algorithm} \label{SE:insights}
To be more precise, recall we aim at maximizing the mapping (and controlling the regret)
\[
F(x)=\sum_{k=1}^K f_k(x_k) \quad \textrm{with } x=(x_1,\ldots,x_K) \in \DK.
\]
As we have a working procedure to handle only $K=2$ resources, we will adopt a divide-and-conquer strategy by diving the mapping $F$ into two sub-mapping  $\F{1}{1}$ and $\F{1}{2}$ defined by
\[
\F{1}{1}(x)=\sum_{k=1}^{\lceil K/2 \rceil} f_k(x_k) \quad \et \quad \F{1}{2}(x)=\sum_{k=\lceil K/2 \rceil+1}^{K} f_k(x_k).
\]
Since the original mapping $F$ is separable, we can reduce the optimization of $F$ over the simplex $\DK$ to the optimization of a sum of two functions over the simplex of dimension 1 (thus going back to the case of $K =2$ resources). Indeed, 
\begin{align*}
\max_{\normun{x}=1}F(x) & =\max_{z \in [0,1]} \left(\max_{\normun{x}=z} \F{1}{1}(x)+\max_{\normun{x}=1-z} \F{1}{2}(x)\right)\\
&\triangleq \max_{z \in [0,1]} \H{1}{1}(z) + \H{1}{2}(1-z)\ .
\end{align*}
Now we aim to apply the machinery of   $K=2$ resources to the reward mappings $ \H{1}{1}$ and $ \H{1}{2}$. The major issue  is that we do not have directly access to the gradients $\nabla \H{1}{1}(z)$ and $\nabla \H{1}{2}(1-z)$ of those functions because they are defined via an optimization problem. However, can apply again the divide-and-conquer approach to $\H{1}{1}$ and compute its gradient using the envelope theorem~\citep{envelope}. Indeed, divide again $ \F{1}{1}$ into the two following mappings  $\F{2}{1}$ and $\F{2}{2}$ defined by
\[
\F{2}{1}(x)=\sum_{k=1}^{\lceil K/4 \rceil} f_k(x_k) \quad \et \quad \F{2}{2}(x)=\sum_{k=\lceil K/4 \rceil+1}^{\lceil K/2 \rceil} f_k(x_k).
\]
Then as above, we can rewrite the optimization problem defining $ \H{1}{1}$  as another optimization problem over $[0,z]$ by noting that
\begin{align*}
\H{1}{1}(z) = \max_{\normun{x}=z} \F{1}{1}(x) & = \max_{\omega \in [0,z]} \left(\max_{\normun{x}=\omega} \F{2}{1}(x)+\max_{\normun{x}=z-\omega} \F{2}{2}(x)\right)\\
&\triangleq \max_{\omega \in [0,z]} \H{2}{1}(\omega) + \H{2}{2}(z-\omega)\ .
\end{align*}
The envelope theorem now gives the following lemma (whose proof is immediate and omitted).
\begin{lemma}\label{Envelop} Let $\omega^*_z \in [0,z]$ be the maximizer of $\H{2}{1}(\omega) + \H{2}{2}(z-\omega)$, then
\[
\nabla \H{1}{1}(z) =\left\{\begin{array}{ll} \nabla \H{2}{1}(\omega^*_z)= \nabla \H{2}{2}(z-\omega^*_z)  & \text{if } \omega^*_z \in (0,z) \\   
 \nabla \H{2}{2}(z)  & \text{if } \omega^*_z =0 \\
  \nabla \H{2}{1}(z)& \text{if } \omega^*_z =z \end{array} \right. .
\]
\end{lemma}
Recall that gradients of  $\H{1}{1}(z)$ and  $\H{1}{2}(1-z)$ were needed to apply the $K=2$ machinery to the optimization of $F$ once this problem is rewritten as $\max_z \H{1}{1}(z)+\H{1}{2}(1-z)$. Lemma \ref{Envelop} provides them, as the gradient of yet other functions $\H{2}{1}$ and/or $\H{2}{2}$. Notice that if $K=4$, then those two functions are actually the two basis functions $f_1$ and $f_2$, so the agent has direct access  to their gradient (up to some noise). It only remains to find the point $\omega^*_z$ which is done with the binary search introduced in the previous section.

If $K>4$, the gradient of $\H{2}{1}$ (and, of course, of $\H{2}{2}$) is not directly accessible, but we can again divide $\H{2}{1}$ into two other functions $\H{3}{1}$ and $\H{3}{2}$. Then the gradient of $\H{2}{1}$ will be expressed, via Lemma~\ref{Envelop}, as gradients of $\H{3}{1}$ and/or $\H{3}{2}$ at some specific point (again, found by binary searches as in $K=2$). We can repeat this process as long as $\H{k}{1}$ and $\H{k}{2}$ are not basis functions in $\mathcal{F}$ and $\cF$ can be ``divided'' to compute recursively the gradients of each $\H{k}{j}$ up to $\H{1}{1}$ and $\H{1}{2}$, up to the noise and some estimation errors that must be controlled.

\subsection{The Generic Algorithm}

A more detailed version of our generic algorithm, with the notations required for the proof can be found in Appendix~\ref{Appendix:K4}. To give some intuitions, consider a binary tree whose root is labeled by the function $F(x) = \sum_{k=1}^K f_k(x_k)\triangleq\F{0}{1}(x)$ that we want to maximize. We are going to label recursively the nodes of this tree by functions and the leaves of this tree are going to be the elements of $\mathcal{F}$. Denote  by $\F{i}{j}$ the function created at the nodes of depth $i$, with $j$ an increasing index from the left to the right of the tree. If $\F{i}{j}(x)$ is not a leaf, \ie~not an element of $\mathcal{F}$, then there exist two indices $k_1 < k_2 $ such that $\F{i}{j}(x)=\sum_{k=k_1}^{k_2} f_k(x_k)$ and we define
\[
\F{i+1}{2j-1}(x)=\sum_{k=k_1}^{\lfloor (k_1+k_2)/2 \rfloor} f_k(x_k) \quad \et \quad \F{i+1}{2j}(x)=\sum_{k=\lfloor (k_1+k_2)/2 \rfloor+1	}^{k_2}f_k(x_k).
\]
If $K$ is not a power of $2$ we can add artificial functions with value $0$ in order to obtain a balanced tree. The optimization of $\F{i}{j}$ can be done recursively since 
\[
\max_{\normun{x}=z_n}\F{i}{j}(x)=\max_{z_{n+1} \in [0,z_n]} \left(\max_{\normun{x}=z_{n+1}} \F{i+1}{2j-1}(x)+\max_{\normun{x}=z_n-z_{n+1}} \F{i+1}{2j}(x)\right).
\]
In order to clarify notations we also define the following mappings as in Section~\ref{SE:insights}:
\begin{definition}
\label{def:gh4}
For every $i$ and $j$ in the constructed binary tree of functions, we define
\begin{align*}
\H{i}{j}(z) \triangleq \max_{\normun{x}=z} \F{i}{j}(x) \quad\et\quad
\G{i}{j}(z;y) \triangleq \H{i+1}{2j-1}(z)+\H{i+1}{2j}(y-z).
\end{align*}
We also note $\Dij(v)$ the binary search whose goal is to optimize the function $\Gij(\point;v)$.
With these notations, it holds that for all $z_n \in [0,1]$,
\[
\H{i}{j}(z_n)=\max_{z_{n+1} \in [0,z_n]} \G{i}{j}(z_{n+1};z_n)=\max_{z_{n+1} \in [0,z_n]} \H{i+1}{2j-1}(z_{n+1})+\H{i+1}{2j}(z_n-z_{n+1}),
\]
\end{definition}
and the gradient of $\H{i}{j}(z)$ can be expressed in terms of those of $\H{i+1}{2j-1}$ and $\H{i+1}{2j}$ by Lemma~\ref{Envelop}.

As a consequence, the gradients of $\H{i}{j}$ can be \textbf{recursively} approximated using estimates of the gradients of their children (in the binary tree). Indeed, assume that one has access to $\varepsilon$-approximations of $\nabla \H{i+1}{2j-1}$ and $\nabla \H{i+1}{2j}$. Then Lemma~\ref{Envelop} directly implies that a $\varepsilon$-approximation of its gradient $\nabla \H{i}{j}(z)$ can be computed by a binary search on $[0,z]$. Moreover, notice that if a binary search is optimizing  $\H{i}{j}$ on $[0,z]$ and is currently querying the point $\omega$, then the level of approximation required (and automatically set to) is equal to $|\nabla \H{i+1}{2j-1}(\omega)- \nabla \H{i+1}{2j}(z-\omega)|$. This is the crucial property that allows a control on the regret.

The main algorithm can now be simply summarized as performing a binary search for the maximization of $\H{1}{1}(z)+\H{1}{2}(1-z)$ using recursive estimates of $\nabla \H{1}{1}$ and $\nabla \H{1}{2}$.

\subsection{Main ideas of  the proof of Theorem~\ref{TH:K4Beta}}
The full analysis of Theorem~\ref{TH:K4Beta} is too involved to be detailed thoroughly here. The detailed proof is provided in Appendix~\ref{Appendix:K4}. We provide nevertheless a very natural intuition in the case of strongly concave mappings or $\beta >2$, as well as the main ingredients of the general proof. 

Recall that in the case where $\beta>2$, the average regret of the algorithm for $K=2$ scales as $\log(T)/T$. As a consequence, running a binary search induces a cumulative regret of the order of $\log(T)$. The generic algorithm is defined recursively over a binary tree of depth $\log_2(K)$ and each function in the tree is defined by a binary search over its children.  So at the end, to perform a binary search over  $\H{1}{1}(z)+\H{1}{2}(1-z)$, the algorithm imbricates  $\log_2(K)$ binary searches to compute gradients. The error made by these binary searches cumulate (multiplicatively) ending up in a cumulative regret term of the order of $\log(T)^{\log_2(K)}$.

For $\beta <2$, the analysis is more intricate, but the main idea is the same one; to compute a gradient, $\log_2(K)$ binary searches must be imbricated and their errors cumulate to give Theorem~\ref{TH:K4Beta}. We give here some of the main ingredients of the proof. As explained in Appendix~\ref{Appendix:K4} we can associate a regret for each binary search, and we call $\Rij(v)$ the regret associated to the binary search $\Dij(v)$. Since we have more than $2$ resources we have to imbricate the binary searches in a recursive manner in order to get access to the gradients of the functions $\Hij$. This will lead to a regret $\Rij(v)$ for the binary search $\Dij(v)$ that will recursively depend on the regrets of the binary searches corresponding to the children (in the tree) of $\Dij(v)$. An important part of the proof of Theorem~\ref{TH:K4Beta} is therefore devoted to proving the following proposition.
\begin{proposition}
\label{prop:subregret}
The regret $\Rij(v)$ of the binary search $\Dij(v)$ is bounded by:
\[
\Rij(v) \leq \sum_{r=1}^{\rmax}8\log(2T/\delta) \dfrac{\abs*{\gij(w_r;v)}}{\abs*{\nabla \gij(w_r;v)}^2} \log(T)^{\log_2(K)-1-i} + \R{i+1}{2j-1}(w_r)+\R{i+1}{2j}(v-w_r),
\]
where $\{w_1, \dots, w_{\rmax} \}$ are the different samples of $\Dij(v)$ and $\gij(\point;v)\doteq\Gij(\point;v)-\max_z \Gij(z;v)$.
\end{proposition}

The goal of the remaining of the proof of Theorem~\ref{TH:K4Beta} is to bound $\R{0}{1}(1)$. The very natural way to do it is to use the previous proposition with the \LL inequality to obtain a simple recurrence relation between the successive values of $\Rij$. The end of the proof is then similar to the proofs done in the case $K=2$. Besides we can note that the statement of Proposition~\ref{prop:subregret} shows clearly that adding more levels to the tree results in an increase of the exponent of the $\log(T)$ factor.


\section{Conclusion} 

We have considered the problem of multi-resource allocation under the classical assumption of diminishing returns. This appears to be a concave optimization problem and we proposed an algorithm based on imbricated binary searches to solve it. Our algorithm is particularly interesting in the sense that it is fully adaptive to all parameters of the problem (strong convexity, smoothness, \LL exponent, etc.). Our analysis provides meaningful upper bound for the regret that matches the lower bounds, up to logarithmic factors. The experiments we conducted (see Appendix~\ref{Appendix:Expe}) validate as expected the theoretical guarantees of our algorithm, as empirically regret seems to decrease polynomially with $T$ with the right exponent.

\acks{X.\ Fontaine was supported by grants from R{\'e}gion
Ile-de-France. This work was supported by a public grant as part of the Investissement d'avenir project, reference ANR-11-LABX-0056-LMH, LabEx LMH, in a joint call with Gaspard Monge Program for optimization, operations research and their interactions with data sciences. V.\ Perchet also acknowledges the support of the ANR under the grant ANR-19-CE23-0026.
}

\bibliography{monotone_bandits}

\newpage

\renewcommand{\theHsection}{A\arabic{section}}
\appendix

\section{Additional results on the \LL inequality}
\label{app:loj}

We give here more detailed about the \LL inequality. In this section we state all the results for convex functions. Their equivalents for concave functions are easily obtained by symmetry. The first one is the fact that every uniformly convex function verifies the \LL inequality.

\begin{definition}
\label{def:loj}
A function $f:\bR^d \to \bR$ satisfies the \LL inequality if
\[
\forall x \in \mathcal{X},\    f(x) - \min_{x^* \in \mathcal{X}} f(x^*)  \leq \mu \| \nabla f(x)\|^\beta .
\]
\end{definition}

\begin{definition}
\label{def:UC}
A function $f:\bR^d\to \bR$ is uniformly-convex with parameters $\rho\geq 2$ and $\mu>0$ if and only if for all $x,y \in \bR^d$ and for all $\alpha \in [0,1]$,
\[
f(\alpha x + (1-\alpha)y) \leq \alpha f(x)+(1-\alpha)f(y)-\dfrac{\mu}{2}\alpha(1-\alpha)\left[\alpha^{\rho-1}+(1-\alpha)^{\rho-1}\right]\norm{x-y}^{\rho}.
\]
\end{definition}
\begin{proposition}
\label{prop:ucloj}
If $f$ is a differentiable $(\rho, \mu)$-uniformly convex function then it satisfies the \LL inequality with parameters $\beta=\rho/(\rho-1)$ and $c=\left(\dfrac{2}{\mu}\right)^{1/(\rho-1)} \dfrac{\rho-1}{\rho^{\rho/(\rho-1)}}$.
\end{proposition}
\begin{proof}
A characterization of differentiable uniformly convex function (see for example~\citep{judistky}) gives that for all $x,y \in \bR^d$
\[
f(y) \geq f(x) + \la \nabla f(x), y-x \ra + \dfrac{1}{2}\mu \norm{x-y}^{\rho}.
\]
Consequently, noting $f(x^*)=\inf f(x)$,
\[
f(x^*) \geq \inf_y \underbrace{\left\lbrace f(x) + \la \nabla f(x), y-x \ra + \dfrac{1}{2}\mu \norm{x-y}^{\rho} \right\rbrace}_{g(y)}.
\]
We now want to minimize the function $g$ which is a strictly convex function. We have
\begin{align*}
\nabla g(y)=\nabla f(x)+ \dfrac{\mu}{2} \rho \norm{x-y}^{\rho-2}(y-x).
\end{align*}
$g$ reaches its minimum for $\nabla g(y)=0$ and $\nabla f(x)=-\dfrac{\mu}{2} \rho \norm{x-y}^{\rho-2}(y-x)$. This gives
\[
f(x^*)\geq f(x)+\dfrac{\mu}{2}\norm{x-y}^{\rho}(1-\rho).
\]
Since $\norm{\nabla f(x)}=\dfrac{\mu\rho}{2}\norm{x-y}^{\rho-1}$ we obtain
\begin{align*}
f(x)-f(x\st) &\leq (\rho-1)\dfrac{\mu}{2}\left(\dfrac{2}{\mu\rho}\norm{\nabla f(x)}\right)^{\rho/(\rho-1)}\\
&\leq \left(\dfrac{2}{\mu}\right)^{1/(\rho-1)} \dfrac{\rho-1}{\rho^{\rho/(\rho-1)}} \norm{\nabla f(x)}^{\rho/(\rho-1)}.
\end{align*}
\end{proof}
In particular a $\mu$-strongly convex function verifies the \LL inequality with $\beta=2$ and $c=1/(2\mu)$.

We now prove a similar link between the Tsybakov Noise condition (TNC) and the \LL equation.

\begin{proposition}
\label{prop:tncloj}
If $f$ is a convex differentiable function locally satisfying the TNC with parameters $\kappa$ and $\mu$ then it satisfies the \LL equation with parameters $\kappa/(\kappa-1)$ and $\mu^{-1/(\kappa-1)}$.
\end{proposition}

\begin{proof}
Let $x,y \in \bR^d$. Since $f$ is convex we have, noting $x^*=\argmin f$,
\begin{align*}
f(y) &\geq f(x)+\la \nabla f(x), y-x \ra \\
f(x)-f(x^*)&\leq \la \nabla f(x), x-x^* \ra \\
f(x)-f(x^*)&\leq \norm{\nabla f(x)}\norm{x-x^*}.
\end{align*}
The TNC gives $f(x)-f(x^*) \geq \mu \norm{x-x^*}^\kappa$, which means that $\norm{x-x^*} \leq \mu^{-1/\kappa}\left(f(x)-f(x^*)\right)^{1/\kappa}$ and consequently,
\begin{align*}
f(x)-f(x^*) &\leq \norm{\nabla f(x)} \mu^{-1/\kappa}\left(f(x)-f(x^*)\right)^{1/\kappa} \\
\left(f(x)-f(x^*)\right)^{1-1/\kappa}&\leq \mu^{-1/\kappa}\norm{\nabla f(x)} \\
\left(f(x)-f(x^*)\right)&\leq \mu^{-1/(\kappa-1)}\norm{\nabla f(x)}^{\kappa/(\kappa-1)}.
\end{align*}
This concludes the proof.
\end{proof}

We now show that the two classes of uniformly convex functions and \LL functions are distinct by giving examples of functions that verify the \LL inequality and that are not uniformly convex.

\begin{example}
The function $f:(x,y) \in \bR^2 \mapsto (x-y)^2$ verifies the \LL inequality but is not uniformly convex on $\bR^2$.
\end{example}
\begin{proof}
$\nabla f(x,y)=2(x-y, y-x)\t$ and $\norm{\nabla f(x,y)}^2=8(x-y)^2=8f(x,y)$. Consequently, since $f$ is minimal at $0$, $f$ verifies the \LL inequality for $\beta=2$ and $c=1/8$.

Let $a=(0,0)$ and $b=(1,1)$. If $f$ is uniformly convex on $\bR^2$ with parameters $\rho$ and $\mu$ then, for $\alpha=1/2$,
\begin{align*}
f(a/2+b/2) &\leq f(a)/2+f(b)/2-\mu/4(2^{1-\rho})\norm{a-b}^{\rho} \\
0 &\leq - \mu/4(2^{1-\rho}) \sqrt{2}^{\rho}.
\end{align*}
This is a contradiction since $\mu>0$ and $\rho \geq 2$.
\end{proof}

\begin{example}
The function $g:(x,y,z) \in \Delta^3 \mapsto (x-1)^2+2(1-y)+2(1-z)$ is not uniformly convex on the simplex $\Delta^3$ but verifies the \LL inequality.
\end{example}

\begin{proof}
$g$ is constant on the set $\{x=0\}$ (since $y+z=1$). And therefore $g$ is not uniformly convex (take two distinct points in $\{x=0\}$).

We have $\nabla g(x,y,z)=(2x-2,-2,-2)\t$ and $\norm{\nabla g(x,y,z)}^2=4((x-1)^2+2) \geq 8$.
Since $y+z=1-x$ on $\Delta^3$, we have $g(x,y,z)=(x-1)^2+4-2(1-x)=x^2+3$. Consequently $\min g=3$.
Hence $g(x,y,z)-\min g = x^2 \leq 1 \leq \norm{\nabla g(x,y,z)}^2$ and $g$ verifies the \LL inequality on $\Delta^3$.
\end{proof}

We conclude this section by giving additional examples of functions verifying the \LL inequality.

\begin{example}\label{ex:alpha}
If $h:x \in \bR^K \mapsto\|x-x\st\|^{\alpha}$ with $\alpha \geq 1$. Then $h$ verifies the \LL inequality with respect to the parameters $\beta=\alpha/(\alpha-1$) and $c=\sqrt{K}$.
\end{example}

The last example is stated in the concave case because it is an important case of application of our initial problem.
\begin{example}
\label{ex:sum}
Let $f_1, \dots, f_K$ be such that $f_k(x)=-a_k x^2+b_k x$ with $b_k \geq 2a_k \geq 0$. Then $F=\sum_k f_k(x_k)$ satisfies the \LL inequality with $\beta=2$ if at least one $a_k$ is positive. Otherwise, the inequality is satisfied on $\DK$ for any $\beta \geq 1$ (with a different constant for each $\beta$).
\end{example}

\begin{proof}
Indeed, let $x \in \DK$.
If there exists at least one positive $a_k$, then $F$ is quadratic, so if we denote by $x^\star$ its maximum and $H$ its Hessian (it is the diagonal matrix with $-a_k$ on coordinate $k$),  we have
\[
F(x)-F(x^\star)=(x-x^\star)^\top H (x-x^\star)\ \text{ and  }\ \nabla F(x) = 2H(x-x^\star).\] 
Hence $F$ satisfies the \LL conditions with $\beta=2$ and $c=1/(4 \min_{k} a_k)$. If all  $f_k$ are linear, then $F(x^*)-F(x)\leq \max_j b_j - \min_j b_j$ and $\|\nabla F(x)\| = \|b\|$. Given any $\beta \geq 1$, it holds that
\[
F(x^*)-F(x) \leq c_\beta \|\nabla F(x)\|^\beta=c_\beta \|b\|^\beta \quad \text{with } c_\beta = ( \max_j b_j - \min_j b_j)/\|b\|^\beta\ .
\]
\end{proof}

\section{Additional results on the complexity class}\label{SE:complex}

In this appendix we want to give more precisions on the class of functions we are considering. We give more intuition and we prove some results on examples of classes satisfying our assumption. Finally we will state some properties of the functions of this tree.

\subsection{Motivations and examples of sets of functions $\cF$ satisfying inductively the \LL inequality}

First, we recall the definition of the local TNC inequality, around the minimum $x^*$  of a function $f$ with vanishing gradient. More precisely, $f$ satisfies \textbf{locally} the TNC if 
\[\forall x \in \mathcal{X}, \quad f(x)- \min_{x^* \in \mathcal{X}}f(x^*) \geq  \mu \| x-x^*\|^\kappa, \]
where in the above the  $x^*$ on the r.h.s.\ is the minimizer of $f$ the closer to $x$ (in case where $f$ has non-unique minimizer).

Uniform convexity, TNC and \LL inequality are connected since it is well known that if a function $f$ is uniformly convex, it satisfies both the local TNC and the \LL inequality. Those two concepts are actually equivalent for convex mappings.

The most precise complexity parameter is therefore induced by the \LL inequality; however, the mappings $f_k$ considered are increasing on $[0,1]$ hence $\nabla f(x^*)$ might be non-zero and the concept of \LL inequality is not appropriate; this is the reason why we need to define the concept of funtions that satisfy \textit{pair-wisely} the \LL inequality (see Subsection~\ref{ssec:complexity}).

We provide now examples of functions that satisfy inductively the \LL inequality. In particular, a set of functions of cardinality $2$ satisfies inductively the \LL inequality if and only if these functions satisfy it pair-wisely. Another crucial property of our construction is that if $f_{\mathrm{left}}$ and $f_{\mathrm{right}}$ are concave, non-decreasing and zero at $0$, then these three properties also hold for  their parent $x \mapsto \max_{z \leq x} f_{\mathrm{left}}(z) + f_{\mathrm{right}}(x-z)$. As a consequence, if these three properties hold at  the leaves, they will hold at all nodes of the tree. See Proposition \ref{PR:Model} for similar alternative statements.

\begin{proposition}\label{PR:Model}
Assume that $\mathcal{F}=\{f_1,\ldots,f_K\}$ is finite then $\mathcal{F}$ satisfies inductively the \LL inequality with respect to  some $\beta_\mathcal{F} \in [1,+\infty)$. Moreover,
\begin{enumerate}
\item if  $f_k$ are all concave, non-decreasing and $f_k(0)=0$, then all functions created inductively in the tree satisfy the same assumption.
\item If $f_k$ are all $\rho$-uniformly concave, then so are all the functions created and $\mathcal{F}$ satisfies inductively the \LL inequality for $\beta_\mathcal{F} \geq \frac{\rho}{\rho-1}$. 
\item If $f_k$ satisfies the  global $\kappa$-TNC, then so are all the functions created and $\mathcal{F}$ satisfies inductively the \LL inequality for $\beta_\mathcal{F} \geq \frac{\kappa}{\kappa-1}$.
\item If $f_k$ satisfies the  global $\beta$-\LL inequality, then so are all the functions created and $\mathcal{F}$ satisfies inductively the \LL inequality for $\beta_\mathcal{F} \geq \beta$. 
\item If $f_k$ are concave, then $\mathcal{F}$ satisfies inductively the \LL inequality w.r.t.\   $\beta_\mathcal{F} = 1$. 
\item If  $f_k$ are linear then $\mathcal{F}$ satisfies inductively the \LL inequality w.r.t.\  any $\beta_\mathcal{F} \geq 1$.
\item More specifically, if $\mathcal{F}$ is a finite subset of the following class of functions \begin{align*}\mathcal{C}_\alpha
&:=\big\{ x \mapsto  \theta (\gamma-x)^\alpha -\theta \gamma^\alpha \, ; \,  \theta \in \dR_-, \gamma \geq 1 \big\}, \quad \text{ if } \alpha > 1
\end{align*}
 then $\mathcal{F}$ satisfies inductively the \LL inequality with respect to   $\beta = \frac{\alpha}{\alpha-1} $.
\end{enumerate}
\end{proposition}

\begin{proof}
\begin{enumerate}
\item We just need to prove that the mapping $x \mapsto H(x)= \max_{z \leq x} f_1(z) + f_2(x-z)=\max_{z \leq x} G(z;x)$ satisfies the same assumption as $f_1$ and $f_2$, the main question being concavity.  Given $x_1, x_2, \lambda \in [0,1]$, let us denote by $z_1$ the point where $G(\point ; x_1)$ attains its maximum (and similarly $z_2$ where $G(\point ; x_2)$ attains its maximum). Then the following holds
\begin{align*}
H(\lambda x_1 +(1-\lambda)x_2) & \geq f_1(\lambda z_1 +(1-\lambda)z_2)+ f_1(\lambda x_1 +(1-\lambda)x_2-\lambda z_1 -(1-\lambda)z_2)\\
& \geq \lambda f_1(z_1) + (1-\lambda)f_1(z_2) + \lambda f_2(x_1-z_1) + (1-\lambda)f_2(x_2-z_2)\\
&= \lambda H(x_1) + (1-\lambda)H(x_2)
\end{align*}
so that concavity is ensured. The fact that $H(0)=0$ and $H(\cdot)$ is non-decreasing are trivial. 

\item Let us prove that the mapping $\left(x \mapsto H(x) = \max{0 \leq z \leq x} f_1(z)+f_2(x-z)\right)$ is also $\rho$-uniformly concave.

Let $\alpha \in (0,1)$. Let $(x,y) \in \bR^2$. Let us denote by $z_x$ the point in $(0,x)$ such that $H(x)=f_1(z_x)+f_2(x-z_x)$ and by $z_y$ the point in $(0,y)$ such that $H(y)=f_1(z_y)+f_2(y-z_y)$. We have
\begin{align*}
\alpha H(x) + (1-\alpha)H(y) &= \alpha f_1(z_x) + \alpha f_2(x-z_x) + (1-\alpha) f_1(z_y) + (1-\alpha) f_2(y-z_y) \\
&\leq f_1(\alpha z_x + (1-\alpha) z_y) - \dfrac{\mu}{2}\alpha(1-\alpha)\left(\alpha^{\rho-1}+(1-\alpha)^{\rho-1}\right)\norm{z_x-z_y}^{\rho} \\
&+ f_2(\alpha (x-z_x) + (1-\alpha) (y-z_y)) \\ &- \dfrac{\mu}{2}\alpha(1-\alpha)\left(\alpha^{\rho-1}+(1-\alpha)^{\rho-1}\right)\norm{x-z_x-y+z_y}^{\rho} \\
&\leq H(\alpha x + (1-\alpha)y) - \dfrac{\mu}{2}\alpha(1-\alpha)(\norm{x-y}/2)^{\rho}
\end{align*}
where we used the fact that $f_1$ and $f_2$ are $\rho$-uniformly concave, and the definition of $H(\alpha x + (1-\alpha)y)$, and that $a^\rho+b^{\rho} \geq ((a+b)/2)^{\rho}$, for $a,b \geq 0$.

This proves that $H$ is $(\rho,\mu/2^{\rho})$-uniformly convex. Finally Proposition~\ref{prop:ucloj} shows that $\cF$ satisfies inductively the \LL inequality for $\beta_F \geq \rho/(\rho-1)$.

\item Let us use the same notations as in the previous proof. We want to show that $H$ satisfies the global TNC equation.

Let us suppose that $x \geq y$. We will show first that $z_x \geq z_y$. Let us consider the functions $G_x:z \mapsto f_1(z)+f_2(x-z)$ and $G_y:z \mapsto f_1(z)+f_2(y-z)$.

\begin{itemize}
\item If $z_y=0$, $z_x\geq z_y$.
\item If $z_y=y$, then $\nabla G_y(y)=\nabla f_1(y)-\nabla f_2(0)\geq 0$. Then $\nabla G_x(y)=\nabla f_1(y)-\nabla f_2(x-y)\geq \nabla f_1(y)-\nabla f_2(0) \geq 0$ since $-\nabla f_2$ is non-decreasing by concavity of $f_2$. Consequently the maximum of $G_x$ is reached for $z \geq y$ and $z_x \geq z_y$.
\item If $z_y \in (0,y)$. Then $\nabla G_x(z_y)=\nabla f_1(z_y)-\nabla f_2(x-z_y) \geq \nabla f_1(z_y)-\nabla f_2(y-z_y)$. Consequently $\nabla G_x(z_y) \geq \nabla G_y(z_y)$ and $\nabla G_y(z_y)=0$. Therefore $z_x \geq z_y$. 
\end{itemize}
We use the exact same proof to show that $x-z_x \geq y-z_y$ (by inverting the roles of $f_1$ and $f_2$).

Using the global $\kappa$-TNC for $f_1$ we get, since it is non-decreasing,
\[
f_1(z_x)-f_1(z_y) \geq \mu \norm{z_x-z_y}^\kappa
\]
and similarly for $f_2$:
\[
f_2(x-z_x)-f_2(y-z_y) \geq \mu \norm{x-z_x-y+z_y}^\kappa.
\]Summing these inequalities gives
\[
H(x)-H(y) \geq \mu/2^\kappa \norm{x-y}^\kappa. 
\]
This shows that $H$ satisfies the global TNC equation for parameters $\kappa$ and $\mu/2^\kappa$.

Proposition~\ref{prop:tncloj} finally shows that $\cF$ inductively satisfies the \LL equation for parameter $\beta_F \geq \kappa/(\kappa-1)$.

\item We still use the same notations as before. We want to show that $H$ satisfies the global \LL equation.

\begin{align*}
\abs{H(x)-H(y)} &= \abs{f_1(z_x)+f_2(x-z_x)-f_1(z_y)-f_2(y-z_y)}\\
&\leq \abs{f_1(z_x)-f_1(z_y)}+\abs{f_2(x-z_x)-f_2(y-z_y)} \\
&\leq \mu \norm{\nabla f_1(z_x)-\nabla f_1(z_y)}^\beta + \mu \norm{\nabla f_2(x-z_x)-\nabla f_2(y-z_y)}^\beta
\end{align*}
In the case where $z_x \notin \{0,x\}$, we have $\nabla H(x)=\nabla f_1(z_x)=\nabla f_2(x-z_x)$. If $z_x=0$, $\nabla H(x)=\nabla f_2(x) > \nabla f_1(0)$ and if $z_x=x$, $\nabla H(x)=\nabla f_1(x) > \nabla f_2(0)$.
Let us suppose (without loss of generality) that $x \geq y$.
\begin{itemize}
\item $z_x=0$, then $z_y=0$ (cf previous item) and $\nabla f_2(x-z_x)=\nabla H(x)$ and $\nabla f_2(y-z_y)=\nabla H(y)$ and consequently $\abs{H(x)-H(y)} \leq \mu \norm{\nabla H(x)-\nabla H(y)}^\beta$.
\item $z_x > 0$ so that $\nabla H(x)=\nabla f_1(z_x)$ and $\nabla H(y) \geq \nabla f_1(z_y)$, meaning that $\norm{\nabla f_1(z_x)-\nabla f_1(z_y)}\leq \norm{\nabla H(y)-\nabla H(x)}$, and a similar analysis shows that $\norm{\nabla f_1(x-z_x)-\nabla f_1(y-z_y)}\leq \norm{\nabla H(y)-\nabla H(x)}$.
\end{itemize}
This means finally that $H$ satisfies the \LL equation with parameters $\beta$ and $2\mu$.

\item This point is actually a direct consequence of the following Lemma \ref{lemma:concav}.
\item If $f_1$ and $f_2$ are linear, then $x\mapsto \max_z \leq x f_1(z)+f_2(x-z)$ is either equal to $f_1$ or to $f_2 $ (depending on which one is the biggest). Hence it is linear.
\item Assume that $f_i =\theta_i(\gamma_i -x)^\alpha - \theta_i \gamma_i^\alpha$ for some parameter $\gamma_i>1$ and $\theta_i <0$. Then easy computations show that $H$ is equal to either $f_1$ or $f_2$ on a small interval near 0 (depending on the size of $\nabla f_i(0)$) and then $H(x) = \theta_0(\gamma_0 -x)^\alpha -c_0$ for some parameters $\theta_0<0$ and $\gamma_0>1$. As a consequence, $H$ is defined piecewisely by functions in $\mathcal{C}_\alpha$, a property that will propagate in the binary tree used in the definition of inductive satisfiability of \LL inequality.

The fact that those functions satisfies the \LL inequality with respect to $\beta = \frac{\alpha}{\alpha-1}$ has already been proved in Example \ref{ex:alpha}.
\end{enumerate}
\end{proof}

\subsection{Some properties of the functions of the tree}
We present now some properties of the functions defined in the labeled tree constructed in the previous section.
We begin by a technical and useful lemma.

\begin{lemma}
\label{lemma:orderz}
Let $f$ and $g$ be two differentiable concave functions on $[0,1]$. For $x \in [0,1]$ define $\phi_x:z \in [0, x] \mapsto f(z)+g(x-z)$. And $z_x \doteq \argmax_{z\in[0,x]} \phi_x(z)$. We have the following results:
\begin{itemize}
\item $\phi_x$ is concave;
\item $\forall \ 0\leq x \leq y \leq 1, \ z_x \leq z_y \ \et\  x-z_x \leq y-z_y.$ In particular the function $x\mapsto z_x$ is $1$-Lipschitz continuous.
\end{itemize}
\begin{proof}
The fact that $\phi_x$ is concave is immediate since $f$ and $g$ are concave functions.

If $0\leq x\leq y\leq 1$, we have $g'(y-z_x) \leq g'(x-z_x)$ since $y-z_x \geq x-z_x$ and $g'$ is non-increasing (because $g$ is concave).
Consequently, $\phi_y'(z_x)=f'(z_x)-g'(y-z_x) \geq \phi_x'(z_x)$.
If $z_x=0$, $z_y \geq z_x$ is immediate. Otherwise, $z_x >0$ and $\phi'(z_x) \geq 0$. This shows that $\phi_y'(z_x) \geq 0$ and consequently, that the maximum $z_y$ of the concave function $\phi_y$ is reached after $z_x$. And $z_y \geq z_x$.

The last inequality is obtained in a symmetrical manner by considering the function $\psi_x:z\in[0,x] \mapsto f(x-z)+g(x)$ whose maximum is reached at $z=x-z_x$. This gives $x-z_x \leq y-z_y$.
\end{proof}
\end{lemma}

We now proove two simples lemmas.
\begin{lemma}
\label{lemma:lip}
If $f$ and $g$ are two concave $L$-Lipschitz continous and differentiable functions, then $H:x \mapsto \max_{z \in [0,x]} f(z)+g(x-z)$ is $L$-Lipschitz continuous.
\end{lemma}

\begin{proof}
With the notations of the previous lemma, we have $H(x)=\phi_x(z_x)$ for all $x \in [0,1]$.

Let $x,y \in [0,1]$. Without loss of generality we can suppose that $x \leq y$. We have
\begin{align*}
\abs{H(x)-H(y)}&=\abs{f(z_x)+g(x-z_x)-f(z_y)-g(y-z_y)}\\
&\leq L\abs{z_x-z_y}+L\abs{x-z_x-(y-z_y}\\
&\leq L(z_y-z_x)+L(y-z_y-x+z_x)\\
&\leq L|y-x|.
\end{align*}
We have used the conclusion of Lemma~\ref{lemma:orderz} in the third line.
\end{proof}

\begin{lemma}
\label{lemma:smooth}
If $f$ and $g$ are two concave $L'$-smooth and differentiable functions, then $H:x \mapsto \max_{z \in [0,x]} f(z)+g(x-z)$ is $L'$-smooth.
\end{lemma}

\begin{proof}
Let $x,y \in [0,1]$. Without loss of generality we can suppose that $x \leq y$. We treat the case where $\phi_x \in (0,x)$ and $\phi_y \in (0,y)$. The other (extremal) cases can be treated similarly. The envelop theorem gives that $\nabla H(x)=\nabla f(z_x)$ and $\nabla H(y)=\nabla f(z_y)$. Therefore $\abs{\nabla H(x)-\nabla H(y)}=\abs{\nabla f(z_x)-\nabla f(z_y)}\leq L'\abs{z_x-z_y}\leq L'\abs{x-y}$ with Lemma~\ref{lemma:orderz}.
\end{proof}

Proposition~\ref{PR:Model} and Lemmas~\ref{lemma:lip} and~\ref{lemma:smooth} show directly the following proposition:
\begin{proposition}
\label{prop:lipsmoothtree}
If the functions $f_1, \dots, f_K$ are concave differentiable $L$-Lipschitz continuous and $L'$-smooth then all functions created in the tree are also concave differentiable $L$-Lipschitz continuous and $L'$-smooth.
\end{proposition}

\subsection{Proof of Proposition~\ref{prop:analytic}}
We begin by proving the following lemma:
\begin{lemma}
\label{lemma:analytic}
If $f$ and $g$ are strictly concave real analytic functions then $H:x \mapsto \max_{0\leq z\leq x} f(z)+g(x-z)$ is also a strictly concave real analytic function.
\end{lemma}
\begin{proof}
The fact that $H$ is strictly concave comes from Proposition~\ref{PR:Model}.
Since $f$ and $g$ are real analytic functions we can write
\[
f(x)=\sum_{n\geq 0}a_n x^n \quad \et \quad g(x)=\sum_{n \geq 0}b_n x^n.
\]
Let us consider the function $\phi_x:z\mapsto f(z)+g(x-z)$ for $z\in [0,x]$.
Now, for all $0\leq z\leq x$, we have
\begin{align*}
\phi_x(z)&=f(z)+g(x-z)\\
&=\sum_{n\geq 0}a_n z^n+\sum_{n\geq0}b_n (x-z)^n \\
&=\sum_{n\geq 0}a_n z^n+\sum_{n\geq0}b_n  \sum_{k=0}^n{n\choose k}x^{n-k}(-1)^k z^k\\
&=\sum_{k\geq 0}a_k z^k+\sum_{k\geq0}\left(\sum_{n\geq k}b_n (-1)^k x^{n-k}\right) z^k\\
&=\sum_{k \geq 0} c_k(x)z^k, 
\end{align*}
with $c_k(x)=a_k+\sum_{n\geq k}b_n (-1)^k x^{n-k}$.

Since $f$ and $g$ are concave, $\phi_x$ is also concave. Let $z_x \doteq \argmax_{z \in [0,x]} \phi_x(z)$. We have $H(x)=\phi_x(z_x)$
If $z_x \in (0,x)$ then $\nabla \phi_x(z_x)=0$ because $\phi_x$ is concave. Consequently $\sum_{k\geq 0} c_{k+1}(x)(k+1)z_x^k =0$.

Let us consider the function $\Psi:(x,z) \mapsto \sum_{k\geq 0} c_{k+1}(x)(k+1)z_x^k = \nabla \phi_x(z)$. Provided that $\nabla_z \Psi(x,z_x)$ is invertible then $z_x$ is unique and is an analytic function of $x$ thanks to the analytic implicit function theorem~\citep{berger}.
Since $f$ and $g$ are strictly concave the invertibility condition is satisfied since $\nabla_z \Psi(x,z_x)=f''(z)+g''(x-z)$, and the result is proved.

\end{proof}

\begin{proof}{of Proposition~\ref{prop:analytic}.}
Let us show that $\cF$ satisfies inductively the \LL inequality. Let $f$ and $g$ be two siblings of the tree defined in Appendix~\ref{SE:complex}. Inductively applying Lemma~\ref{lemma:analytic} shows that $(x\mapsto \max_{0\leq z\leq x} f(z)+g(x-z)$ is a strictly concave real analytic function. Since a real analytic function verifies the \LL inequality~\citep{loja}, the result is proved. We set $\beta$ to be the maximum of all \LL exponents in the tree.
\end{proof}


\section{Analysis of the algorithm with $K=2$ resources}\label{Appendix:K2}

\subsection{Proof of Theorem \ref{TH:K2Strongly}}
\begin{proof}
Let $j \in [\jm]$. By concavity of $g$, we have that $-g(x_j) \leq \abs{g'(x_j)}\abs{x\st-x_j}$. Since $g$ is negative, this means that $\abs{g(x_j)}\leq \abs{g'(x_j)}\abs{x\st-x_j}$.

Since $g$ is of class $\mathcal{C}^2$ and $\alpha$-strongly concave,
\begin{align*}
\la g'(x_j) - g'(x\st) | x_j-x\st \ra &\leq -\alpha \norm{x_j-x\st}^2 \\
-\alpha \norm{x_j-x\st}^2 &\geq \la g'(x_j) - g'(x\st) | x_j-x\st \ra \geq - \abs{g'(x_j)}\norm{x_j-x\st} \\
\abs{g'(x_j)} &\geq \alpha \norm{x_j-x\st}.
\end{align*}
Then
\[
\dfrac{\abs{g(x_j)}}{g'(x_j)^2}\leq\dfrac{\abs{g'(x_j)}\abs{x\st-x_j}}{g'(x_j)^2}=\dfrac{\abs{x\st-x_j}}{\abs{g'(x_j)}}\leq\dfrac{1}{\alpha}.
\]

Consequently we have
\[
R(T) \leq \dfrac{\jm}{T\alpha}.
\]

We have for all $j \in [\jm]$, $N_j = 2\log(2T/\delta)\dfrac{1}{g'(x_j)^2}$.
Then
\begin{align*}
T&=8\log(2T/\delta)\sum_{j=1}^{\jm} \dfrac{1}{g'(x_j)^2} \\
&\geq8\log(2T/\delta)\sum_{j=1}^{\jm}\dfrac{1}{L'^2(x_j-x\st)^2} \\
&\geq 8\log(2T/\delta)\dfrac{1}{L'^2(x_{\jm}-x\st)^2} \\
&\geq 8\log(2T/\delta)\dfrac{4^{\jm}}{L'^2}.
\end{align*}
where we used the fact that $g'$ is $L'$-Lipschitz continuous.
Therefore $\jm \leq \log_4\left(\dfrac{TL'^2}{8\log(2T/\delta)}\right)\lesssim \log(T)$.
And finally \[R(T)=\bigo\left(\dfrac{1}{\alpha}\dfrac{\log(T)}{T}\right).\]
\end{proof}

\subsection{Proof of Theorem \ref{TH:K2Beta}, when $\beta >2$}

\begin{proof}
Let $x\in[0,1]$. We know that $\abs{g(x)} \leq c \abs{g'(x)}^{\beta}$.

Then $\dfrac{1}{\abs{g'(x)}^2}\leq \dfrac{c^{2/\beta}}{\abs{g(x)}^{2/\beta}}$, and $\dfrac{\abs{g(x)}}{\abs{g'(x)}^2}\leq c^{2/\beta} \abs{g(x)}^{1-2/\beta}$.

Since $g$ is $L$-Lipschitz on $[0,1]$, we have $\abs{g(x)-g(x\st)}\leq L\abs{x-x\st}$.
Since $g(x\st)=0$ then $\dfrac{\abs{g(x)}}{\abs{g'(x)}^2}\leq c^{2/\beta} L^{1-2/\beta}\abs{x-x\st}^{1-2/\beta}$.

For $j \in [\jm]$, $\dfrac{\abs{g(x_j)}}{\abs{g'(x_j)}^2}\leq c^{2/\beta} L^{1-2/\beta}\left(\dfrac{1}{2^{1-2/\beta}}\right)^{j}$, because $\abs{x\st-x_j}\leq 2^{-j}$, as a consequence of the binary search.
Since $1-2/\beta>0$, 
\[
\sum_{j=1}^{\jm}\left(\dfrac{1}{2^{1-2/\beta}}\right)^{j} < \dfrac{1}{1-2^{2/\beta-1}}.
\]
Finally we have, using that $\delta=2/T^2$,
\begin{align*}
R(T)&=\dfrac{8}{T}\log(2T/\delta)\sum_{j=1}^{\jm}\dfrac{\abs{g(x_j)}}{\abs{g'(x_j)}^2} \\
&\leq\dfrac{24 c^{2/\beta} L^{1-2/\beta}}{1-2^{2/\beta-1}}\dfrac{\log(T)}{T}.
\end{align*}

\end{proof}

\subsection{Proof of Theorem \ref{TH:K2Beta}, when $\beta <2$}

\begin{proof}
We know that
\begin{align*}
R(T)&=\dfrac{1}{T}\sum_{j=1}^{\jm}\abs{g(x_j)}N_j \\
&=8\log(2T/\delta) \dfrac{1}{T}\sum_{j=1}^{\jm} \dfrac{\abs{g(x_j)}}{h_j^2} \\
&\leq 8\log(2T/\delta) \dfrac{1}{T}\sum_{j=1}^{\jm} \dfrac{\abs{g(x_j)}}{g'(x_j)^2}.
\end{align*}
where $h_j \geq g_j$ is such that $N_j=\dfrac{8\log(2/\delta)}{h_j^2}$.
We note
\[
R \doteq \dfrac{TR(T)}{8\log(2T/\delta)}=\sum_{j=1}^{\jm}\dfrac{\abs{g(x_j)}}{h_j^2}.
\]
By hypothesis, $\forall x \in [0,1], \ \abs{g(x)}\leq c \abs{g'(x)}^{\beta}$.
Moreover Lemma~\ref{lemma:concav} gives $\abs{g(x_j)}\leq \abs{g'(x_j)}\abs{x_j-x\st}\leq \abs{g'(x_j)}2^{-j}$.

If we note $g_j \doteq \abs{g'(x_j)}$ we obtain
\[
R \leq \sum_{j=1}^{\jm}\min\left(cg_j^{\beta},\dfrac{g_j}{2^j}\right)\dfrac{1}{h_j^2}.
\]
Let us now note \[T'\doteq \dfrac{T}{8\log(2T/\delta)}.\] We have the constraint \[T'=\sum_{j=1}^{\jm}\dfrac{1}{h_j^2}.\]
Our goal is to bound $R$. In order to do that, one way is to consider the functional \[\cF:(g_1,\dots,g_{\jm}) \in {\bRpp}^{\jm} \mapsto \sj \min\left(cg_j^{\beta},\dfrac{g_j}{2^j}\right)/h_j^2\] and to maximize it under the constraints \[T'=\sum_{j=1}^{\jm}\dfrac{1}{h_j^2} \quad \et \, g_j \leq h_j.\]
Therefore the maximum of the previous problem is smaller than the one of maximizing \[\hat{\cF}:(h_1,\dots,h_{\jm}) \in {\bRpp}^{\jm} \mapsto \sj \min\left(ch_j^{\beta-2},\dfrac{1}{h_j 2^j}\right)\] and to maximize it under the constraints \[T'=\sum_{j=1}^{\jm}\dfrac{1}{h_j^2}.\]
For the sake of simplicity we identify $g_j$ with $h_j$.
The maximization problem can be done with Karush-Kuhn-Tucker conditions: introducing the Lagrangian \[\sL(g_1,\dots,g_{\jm},\lambda)=\cF(g_1,\dots,g_{\jm})+\lambda\left(T'-\sum_{j=1}^{\jm}\dfrac{1}{h_j^2}\right)\] we obtain
\[
\dfrac{\partial \sL}{\partial g_j}=
\begin{cases}
c(\beta-2)g_j^{\beta-3}+\dfrac{2\lambda}{g_j^3}, & \textrm{if } g_j < \hgj \\
-\dfrac{1}{2^j g_j}+\dfrac{2\lambda}{g_j^3}, & \textrm{if } g_j > \hgj
\end{cases}\quad,\textrm{where } \hgj=\left(\dfrac{1}{2^jc}\right)^{1/(\beta-1)}.
\]
$\hgj$ is the point where the two quantities in the $\min$ are equal.
And finally
\[
\begin{cases}
g_j=\left(\dfrac{2\lambda}{c(2-\beta)}\right)^{1/\beta}, & \textrm{if } g_j < \hgj \\
g_j = 2\lambda \cdot 2^j, & \textrm{if } g_j > \hgj.
\end{cases}
\]
We note $\cJ_1\doteq \{j \in [\jm], g_j>\hgj \}$ and $\cJ_2\doteq \{j \in [\jm], g_j<\hgj \}$.
We have
\[
\FF=\underbrace{\sJ \dfrac{1}{2^j g_j}}_{\cF_1}+\underbrace{\sJJ c g_j^{\beta-2}}_{\cF_2}.
\]
We note as well \[T_1\doteq \sJ \dfrac{1}{g_j^2} \quad \textrm{and} \quad T_2 \doteq \sJJ \dfrac{1}{g_j^2} \quad \textrm{such that }T'=T_1+T_2.\]

\noindent{}
\textbf{on $\cJ_2$:}

Since $g_j < \hgj$ on $\cJ_2$, noting $g_2\doteq \left(\dfrac{2\lambda}{c(2-\beta)}\right)^{1/\beta}=g_j$,
\begin{align*}
T_2=\sJJ \dfrac{1}{g_j^2}=\abs{\cJ_2}\dfrac{1}{g_2^2}>\abs{\cJ_2}\dfrac{1}{\hgj^2} \textrm{ for all } j \in \cJ_2
\end{align*}
In particular,
\begin{align*}
T' \geq T_2 > \abs{\cJ_2} \left(\dfrac{1}{c^2 4^{\jdm}}\right)^{-1/(\beta-1)} \geq \abs{\cJ_2} \left(c^2 4^{\abs{\cJ_2}}\right)^{1/(\beta-1)}\geq \left(4^{\abs{\cJ_2}}\right)^{1/(\beta-1)}
\end{align*}
because $c$ can be chosen greater than $1$.
This gives $\abs{\cJ_2} \leq \dfrac{\beta-1}{\log(4)}\log(T)$.

And we know that \[T_2=\sJJ \dfrac{1}{g_j^2}=\abs{\cJ_2}\left(\dfrac{2\lambda}{c(2-\beta)}\right)^{-2/\beta}.\] This gives
\[
\dfrac{2\lambda}{c(2-\beta)}=\left(\dfrac{T_2}{\abs{\cJ_2}}\right)^{-\beta/2}.
\]
We can now compute the cost of $\cJ_2$:
\begin{align*}
\cF_2&=\sJJ c g_j^{\beta-2} \\
&=\abs{\cJ_2}c\left(\dfrac{2\lambda}{c(2-\beta)}\right)^{(\beta-2)/\beta} \\
&=\abs{\cJ_2}c\left(\dfrac{T_2}{\abs{\cJ_2}}\right)^{1-\beta/2} \\
&=cT_2^{1-\beta/2}\abs{\cJ_2}^{\beta/2} \\
&\leq cT_2^{1-\beta/2}\left(\dfrac{\beta-1}{\log(4)}\log(T')\right)^{\beta/2} \\
&\lesssim cT'\left(\dfrac{\log(T')}{T'}\right)^{\beta/2}.
\end{align*}

\textbf{on $\cJ_1$:}

We know that $\forall j \in \cJ_1, \ g_j=2\lambda \ 2^j$.
This gives
\begin{align*}
T_1&=\sJ \dfrac{1}{g_j^2}=\dfrac{1}{4\lambda^2}\sJ\dfrac{1}{4^j} \\
2\lambda&=\sqrt{\dfrac{\sJ 4^{-j}}{T_1}}\\
2\lambda&\leq\sqrt{\dfrac{4\cdot4^{-\jum}}{3T_1}}.
\end{align*}
Since $j \in \cJ_1$, we know that $g_j \geq \hgj$ and $2\lambda\ 2^j \geq \left(\dfrac{1}{2^j c}\right)^{1/(\beta-1)}$, and $2\lambda \geq c^{-1/(\beta-1)}(2^j)^{-\beta/(\beta-1)}$.
With $j=\jum$ we obtain
\begin{align*}
c^{-1/(\beta-1)}(2^{\jum})^{-\beta/(\beta-1)}&\leq\sqrt{\dfrac{4\cdot4^{-\jum}}{3T_1}} \\
\dfrac{\sqrt{3}}{2}\left(2^{\jum}\right)^{-1/(\beta-1)}c^{-1/(\beta-1)}&\leq\dfrac{1}{\sqrt{T_1}} \\
c^{-2}4^{-\jum} &\lesssim T_1^{1-\beta}.
\end{align*}
And we have
\begin{align*}
\cF_1&=\sJ \dfrac{1}{2^j \ 2 \lambda \ 2^j}=\dfrac{1}{2\lambda}\sJ \dfrac{1}{4^j}=2\lambda T_1 \\
&\lesssim \sqrt{T_1} 2^{-\jum} \\
&\lesssim cT_1^{1-\beta/2}
\lesssim cT'^{1-\beta/2}.
\end{align*}
Finally we have shown that $R \lesssim cT'\left(\dfrac{\log(T')}{T'}\right)^{\beta/2}$ and consequently
\begin{align*}
\dfrac{TR(T)}{8\log(2T/\delta)}&\lesssim c\dfrac{T}{8\log(2T/\delta)}\left(\dfrac{\log(T')}{T'}\right)^{\beta/2} \\
R(T)&\lesssim c\left(8\log(2T/\delta)\right)^{\beta/2}\left(\dfrac{\log(T)}{T}\right)^{\beta/2}.
\end{align*}
And using the fact that $\beta<2$ and $\delta=2/T^2$, we have
\[
R(T)\lesssim c\left(\dfrac{\log(T)^2}{T}\right)^{\beta/2}.
\]
\end{proof}

\subsection{Proof of Theorem \ref{TH:K2Beta}, when $\beta=2$}
\begin{proof}
As in the previous proof, we want to bound
\[
R=\sum_{j=1}^{\jm}\dfrac{\abs{g(x_j)}}{g'(x_j)^2}\leq \sj \min\left(c, \dfrac{1}{g_j 2^j}\right).
\]
Let us note $\hgj\doteq \dfrac{1}{c2^j}$, we have to distinguish two cases:
\[
\begin{cases}
\textrm{if }g_j>\hgj, & \textrm{then } \min\left(c, \dfrac{1}{g_j 2^j}\right)=\dfrac{1}{2^j g_j} \\
\textrm{if }g_j<\hgj, & \textrm{then } \min\left(c, \dfrac{1}{g_j 2^j}\right)=c.
\end{cases}
\]
We note $\cJ_1\doteq \{j \in [\jm], g_j>\hgj \}$ and $\cJ_2\doteq \{j \in [\jm], g_j<\hgj \}$.

We have
\[
R\leq\underbrace{\sJ \dfrac{1}{2^j g_j}}_{R_1}+\underbrace{\sJJ c}_{R_2}.
\]
We note as well \[T_1\doteq \sJ \dfrac{1}{g_j^2} \quad \textrm{and} \quad T_2 \doteq \sJJ \dfrac{1}{g_j^2} \quad \textrm{such that }T'=T_1+T_2.\]

\noindent{}
\textbf{on $\cJ_2$:}

\[T_2=\sJJ \dfrac{1}{g_j^2} > \sJJ \dfrac{1}{\hgj^2} \geq \sJJ c^2 4^j \geq 4^{\jdm}.\]
Which gives $\jdm \leq \log(T)$.
Finally, \[R_2=\sJJ c \leq c\jdm \leq c \log(T).\]

\noindent{}
\textbf{on $\cJ_1$:}

We want to maximize $R_1=\displaystyle \sJ \dfrac{1}{2^jg_j}$ under the constraint $T_1=\displaystyle \sJ \dfrac{1}{g_j^2}$.

Karush-Kuhn-Tucker conditions give the existence of $\lambda>0$ such that for all $j \in \cJ_1$, $g_j=2\lambda\cdot 2^j$. As in the previous proof this shows that $R_1=2\lambda T_1$. We can show as well that, if $j \in \cJ_1$,
\[
2\lambda \leq \dfrac{2}{\sqrt{3}}\dfrac{2^{-\jum}}{\sqrt{T_1}}.
\]
And since $j \in \cJ_1$, $g_j > \dfrac{1}{c2^j}$ and then $2\lambda \ 2^j > \dfrac{1}{c2^j}$ which means
\[
2\lambda>\dfrac{1}{c4^{\jum}}.
\]
Putting these inequalities together gives
\[
\sqrt{T_1}\leq \dfrac{2c}{\sqrt{3}}2^{\jum}.
\]
Finally,
\[
R_1=2\lambda T_1\leq\dfrac{2}{\sqrt{3}}\dfrac{2^{-\jum}}{\sqrt{T_1}}T_1\leq \dfrac{4c}{3}.
\]

This shows that
\[
R(T)\lesssim c\log(2T/\delta)\dfrac{\log(T)}{T} \lesssim c \dfrac{\log(T)^2}{T}.
\]
\end{proof}

\subsection{Proof of Theorem \ref{TH:lower}}
\begin{proof}
The proof is very similar to the one of \cite{Shamir13} (see also \citep{BacPer16}) so we only provide the main different ingredients.

Given $T$ and $\beta$, we are going to construct 2 pairs of functions $f_1, f_2$ and $\widetilde{f}_1, \widetilde{f}_2$ such that 
\[
\|f_i -\widetilde{f}_i\|_\infty \leq \frac{c_\beta}{\sqrt{T}} \quad \text{and } \ \|\nabla f_i -\nabla\widetilde{f}_i\|_\infty \leq \frac{c_\beta}{\sqrt{T}}\ .
\]
As a consequence, using only $T$ samples\footnote{Formally, we just need to control the $\ell_\infty$ distance between the gradients, as we assume that the feedbacks of the decision maker are noisy gradients. But we could have assumed that he also observes noisy evaluations of $f_1(x_1)$ and $f_2(x_2)$. This is why we also want to control the $\ell_\infty$ distance between the functions $f_i$ and $\widetilde{f}_i$.}, it is impossible to distinguish between the pair $f_1,f_2$ and the pair $\widetilde{f}_1,\widetilde{f}_2$. And the regret incurred by any algorithm is then lower-bounded (up to some constant) by $$\min_x \max\{ g\st-g(x)\, ; \, \widetilde{g}\st-\widetilde{g}(x)\}$$ where we have defined $g(x) = f_1(x)+f_2(1-x)$ and $g\st= \max_x g(x)$ and similarly for $\widetilde{g}$.

\medskip To define all those functions, we first introduce  $g$ and $\widetilde{g}$  defined as follows, where $\gamma$ is a parameter to be fixed later.
\begin{align*}
g:x\mapsto 
&\begin{cases}
-x^{\beta/(\beta-1)} \quad \textrm{if } x \leq \gamma \\
- \dfrac{\beta}{\beta-1}\gamma^{1/(\beta-1)}x+\dfrac{1}{\beta-1}\gamma^{\beta/(\beta-1)}\quad\textrm{otherwise}
\end{cases}
\\
&\et \\
\widetilde{g}:x\mapsto
&
\begin{cases}
- \abs{x-\gamma}^{-\beta/(\beta-1)} \quad \textrm{if } x \leq 2\gamma \\
- \dfrac{\beta}{\beta-1}\gamma^{1/(\beta-1)}x + \dfrac{\beta+1}{\beta-1}\gamma^{\beta/(\beta-1)}\quad\textrm{otherwise.}
\end{cases}
\end{align*}
The functions have the form of Proposition~\ref{PR:Model} near $0$ and then are linear with the same slope.
Proposition~\ref{PR:Model} ensures that $g_1$ and $g_2$ verify the \LL inequality for the parameter $\beta$. The functions $g_1$ and $g_2$ are concave non-positive functions, reaching their respective maxima at $0$ and $\gamma$.

We also introduce a third function $h$ defined by
\begin{align*}
h:x\mapsto 
&\begin{cases}
(\gamma-x)^{\beta/(\beta-1)}-x^{\beta/(\beta-1)} \quad \textrm{if }\ \frac{\gamma}{2} \leq x \leq \gamma \\
 2\dfrac{\beta}{\beta-1}(\frac{\gamma}{2})^{1/(\beta-1)}(\frac{\gamma}{2}-x)\quad\textrm{if }\ x \leq \frac{\gamma}{2}
 \\
- \dfrac{\beta}{\beta-1}\gamma^{1/(\beta-1)}x+\dfrac{1}{\beta-1}\gamma^{\beta/(\beta-1)}\quad\textrm{if }\ x \geq \gamma 
\end{cases}
\end{align*}

The functions $f_i$ and $\widetilde{f}_i$ are then defined as 
\begin{align*}
f_1(x)=0 \quad & \text{ and } \ \widetilde{f}_1(x)= \widetilde{g}(x)-g(x)+h(x)-\widetilde{g}(0)-g(0)+h(0) \\
 f_2(x) = g(1-x)-g(1) \quad & \text{ and } \ \widetilde{f}_2(x)=g(1-x)-h(1-x)-g(1)+h(1)
\end{align*}
It immediately follows that $f_1(x)+f_2(1-x)$ is equal to $g(x)$ and similarly $\widetilde{f}_1(x) + \widetilde{f}_2(1-x)$ is equal to $\widetilde{g}(x)$ (both up to some additive constant).
\medskip

We observe that for all $x\in[0,1]$:
\begin{align*}
\nabla g(x)=&
\begin{cases}
- \dfrac{\beta}{\beta-1}x^{1/(\beta-1)} \quad \textrm{if } x \leq \gamma \\
- \dfrac{\beta}{\beta-1}\gamma^{1/(\beta-1)}\quad\textrm{otherwise}
\end{cases} \\
&\et \\
\nabla \widetilde{g}(x)=&
\begin{cases}
-\dfrac{\beta}{\beta-1}\sign(x-\gamma)\abs{x-\gamma}^{1/(\beta-1)} \quad \textrm{if } x \leq 2\gamma \\
- \dfrac{\beta}{\beta-1}\gamma^{1/(\beta-1)}\quad\textrm{otherwise}
\end{cases}
\end{align*}
Similarly, we can easily compute the gradient of $h$:
\begin{align*}
\nabla h(x)=
&\begin{cases}
-\frac{\beta}{\beta-1}\Big( (\gamma-x)^{1/(\beta-1)}+x^{1/(\beta-1)}\Big) \quad \textrm{if }\ \frac{\gamma}{2} \leq x \leq \gamma \\
 -2\dfrac{\beta}{\beta-1}(\frac{\gamma}{2})^{1/(\beta-1)}\quad\textrm{if }\ x \leq \frac{\gamma}{2}
 \\
- \dfrac{\beta}{\beta-1}\gamma^{1/(\beta-1)}\quad\textrm{if }\ x \geq \gamma 
\end{cases}
\end{align*}

We want to bound $\normi{\nabla g-\nabla \widetilde{g}}$ as it is clear that $\normi{\nabla h}\leq  \frac{\beta}{\beta-1}\gamma^{1/(\beta-1)}$.

\begin{itemize}
\item For $x \leq \gamma$,
\begin{align*}
\abs*{\nabla g(x)-\nabla \widetilde{g}(x)}&=\dfrac{\beta}{\beta-1}\abs*{-x^{1/(\beta-1)}-(\gamma-x)^{1/(\beta-1)}}\\
&=\dfrac{\beta}{\beta-1}\abs*{x^{1/(\beta-1)}+(\gamma-x)^{1/(\beta-1)}}\\
&\leq\dfrac{\beta}{\beta-1}\left(x^{1/(\beta-1)}+(\gamma-x)^{1/(\beta-1)}\right)\\
&\leq 2\dfrac{\beta}{\beta-1}\gamma^{1/(\beta-1)}.
\end{align*}
\item For $\gamma\leq x\leq2\gamma$, 
\begin{align*}
\abs*{\nabla g(x)-\nabla \widetilde{g}(x)}&=\dfrac{\beta}{\beta-1}\abs*{(x-\gamma)^{1/(\beta-1)}-x^{1/(\beta-1)}}\\
&\leq\dfrac{\beta}{\beta-1}\abs*{(x-\gamma)^{1/(\beta-1)}}+ \abs*{x^{1/(\beta-1)}}\\
&\leq(1+2^{1/(\beta-1)})\dfrac{\beta}{(\beta-1)}\gamma^{1/(\beta-1)}
\end{align*}
\item For $x \geq 2\gamma$, $\abs*{\nabla g(x)-\nabla \widetilde{g}(x)}=0$.
\end{itemize}
Finally we also have that $\normi{\nabla g-\nabla \widetilde{g}}\lesssim \gamma^{1/(\beta-1)}$ , where the notation $\lesssim$ hides a multiplicative constant factor.

Combining the control on $\normi{\nabla g-\nabla \widetilde{g}}$ and $\normi{\nabla h}$, we finally get that 
\[
\normi{\nabla f_1-\nabla \widetilde{f}_1}\lesssim \gamma^{1/(\beta-1)}  \quad \text{ and } \ \normi{\nabla f_2-\nabla \widetilde{f}_2} \lesssim \gamma^{1/(\beta-1)}.
\]
As a consequence, the specific choice of $\gamma=T^{(1-\beta)/2}$ ensures that $\gamma^{1/(\beta-1)} \leq 1/\sqrt{T}$ and thus the mappings $f_i$ are indistinguishable from $\widetilde{f}_i$,

Finally, we get \[
R(T) \geq T \min_x \max(\abs{g(x)},\abs{\widetilde{g}(x)}) \geq T g(\gamma/2) \gtrsim \gamma^{\beta/(\beta-1)} \gtrsim T^{-\beta/2}.
\]
\end{proof}


\section{Analysis of the algorithm with $K>2$ resources} \label{Appendix:K4}

In this section we present a detailed description of the maximization algorithm introduced in Section~\ref{sec:k4} as well as a thorough analysis of its complexity.

\subsection{Detailed description of the algorithm}

The goal of the algorithm is to maximize the following function on the simplex $\Delta^K$:
\[
F(x)=\sum_{k=1}^K f_k(x_k) \quad \mbox{with }x=(x_1,\dots,x_K)\in \Delta^K.
\]
As mentioned in the main text, the idea is to use a divide and conquer strategy in order to be able to use the procedure of $K=2$ resources explained in Section~\ref{sec:k2}. The overall idea is to separate arms recursively into two bundles, creating the aforementioned tree whose root is $F$ and whose leaves are the $f_k$. We explain in this section the algorithm with more details, introducing the relevant definitions and notations for the proof.  

We will denote  by $\F{i}{j}$ the function created at the nodes of depth $i$, with $j$ an increasing index from the left to the right of the tree; in particular $\F{0}{1}=F=\sum_{k=1}^K f_k(x_k)$. This is the function we want to maximize.

\begin{definition}\label{Def:Tree}
Starting from $\F{0}{1}=F=\sum_{k=1}^K f_k(x_k)$, the functions $\F{i}{j}$ are constructed inductively as follows. If  $\F{i}{j}(x)=\sum_{k=k_1}^{k_2} f_k(x_k)$ is not a leaf (\ie~$k_1 < k_2$) we define
\[
\F{i+1}{2j-1}(x)=\sum_{k=k_1}^{\lfloor (k_1+k_2)/2 \rfloor} f_k(x_k) \quad \et \quad \F{i+1}{2j}(x)=\sum_{k=\lfloor (k_1+k_2)/2 \rfloor+1	}^{k_2}f_k(x_k).
\]
\end{definition}

The optimization of $\F{i}{j}$ can be done recursively since 
\[
\max_{\normun{x}=z_n}\F{i}{j}(x)=\max_{z_{n+1} \in [0,z_n]} \left(\max_{\normun{x}=z_{n+1}} \F{i+1}{2j-1}(x)+\max_{\normun{x}=z_n-z_{n+1}} \F{i+1}{2j}(x)\right).
\]
The recursion ends at nodes that are parents of leaves, where the optimization problem is reduced to the case of $K=2$ resources studied in the previous section.

For the sake of notations, we introduce the following functions.

\begin{definition}
\label{def:gh}
For every $i$ and $j$ in the constructed binary tree of functions,
\begin{align*}
\H{i}{j}(z) \triangleq \max_{\normun{x}=z} \F{i}{j}(x) \quad\et\quad
\G{i}{j}(z;y) \triangleq \H{i+1}{2j-1}(z)+\H{i+1}{2j}(y-z).
\end{align*}
With these notations, it holds that for all $z_n \in [0,1]$,
\[
\H{i}{j}(z_n)=\max_{z_{n+1} \in [0,z_n]} \G{i}{j}(z_{n+1};z_n)=\max_{z_{n+1} \in [0,z_n]} \H{i+1}{2j-1}(z_{n+1})+\H{i+1}{2j}(z_n-z_{n+1}).
\]
\end{definition}

The computation of $\H{i}{j}(z_n)$ is made with similar techniques than in the case $K=2$ of Section~\ref{sec:k2}. The idea will be to imbricate several binary searches to get estimates of the functions $\nabla \H{i}{j}$.

More precisely, to maximize the function $\left(u \mapsto \G{i}{j}(u;z_{n})\right)$ a binary search is run over $[0,z_n]$, starting at $u_1=z_n/2$:

\begin{definition}
We note $\Dij(v)$ the binary search run to maximize $\left(w \mapsto \G{i}{j}(w;v)\right)$. We define ${z\st}^{(i)}_j(v)$ as $\argmax \ \Gij(\cdot\,;v)$ and we also call $\Tij(v)$ the total number of queries used by $\D{i}{j}(v)$.
\end{definition}

Inductively, the binary search $\Dij(v)$ searches on the left or on the right of $u_m$, depending on the sign of $\nabla \G{i}{j}(u_m;z_n)$. As it holds that, by definition, $\nabla \Gij(u_m;z_n) =\nabla \H{i+1}{2j-1}(u_m)-\nabla \H{i+1}{2j}(z_n-u_m)$, we need to further estimate $\nabla \H{i+1}{2j-1}(u_m)$ and $\nabla \H{i+1}{2j}(z_n-u_m)$.

This is done using the following properties relating the different gradients of $\H{i}{j}$. They are direct consequences of the envelop theorem (see Lemma~\ref{Envelop}) because
\[
\Hij(u)=\max_{\normun{x}=u} \Fij(x)=\max_{v \in [0,u]} \H{i+1}{2j-1}(v)+\H{i+1}{2j}(u-v).
\]
As a consequence, if $v\st$ denotes  the point where the maximum is reached  and it belongs to $(0,u)$, then
\[
\label{eq:grad}
\nabla \Hij(u)=\nabla \H{i+1}{2j-1}(v\st)=\nabla \H{i+1}{2j}(u-v\st). \numberthis
\]
If $v\st = 0$, then $\nabla \Hij(u)=\nabla \H{i+1}{2j}(u-v\st)$, \ie~the first inequality might only be an inequality. On the other hand, if $v\st = u$, then $\nabla \Hij(u)=\nabla \H{i+1}{2j-1}(v\st)$. This is the central tool that will let us compute the gradients of all the nodes of the tree, from the leaves to the root.
\bigskip

Thanks to the envelop theorem and Equation~\eqref{eq:grad} we are able to compute the gradients $\nabla \Gij(v;u)$ for all nodes in the tree. This is done recursively by imbricating dichotomies.

The goal of the binary searches $\D{i+1}{2j-1}(v)$ and $\D{i+1}{2j}(u-v)$ is to compute an approximate value of $\nabla \Gij(v;u)$. Indeed we have
\[
\nabla \Gij(v;u) = \nabla \H{i+1}{2j-1}(v)-\nabla \H{i+1}{2j}(u-v),
\]
and to compute $\H{i+1}{2j-1}(v)$ (respectively $\nabla \H{i+1}{2j}(u-v)$) we need to run the binary search $\D{i+1}{2j-1}(v)$ (respectively $\D{i+1}{2j}(u-v)$).
Let us denote by $\hnabla \Gij(v;u)$ the approximate value of $\nabla \Gij(v;u)$ computed at the end of the binary searches $\D{i+1}{2j-1}(v)$ and $\D{i+1}{2j}(u-v)$, that compute themselves $\hnabla \H{i+1}{2j-1}(v)$, approximation of $\nabla \H{i+1}{2j-1}(v)$ and $\hnabla \H{i+1}{2j}(u-v)$, approximation of $\nabla \H{i+1}{2j}(u-v)$.

The envelop theorem gives that $\nabla \H{i+1}{2j-1}(v)=\nabla \H{i+2}{4j-3}(w\st)=\nabla \H{i+2}{4j-2}(v-w\st)$ where $w\st=\argmax \G{i+1}{2j-1}(w;v)$.
Therefore in order to compute $\hnabla \H{i+1}{2j-1}(v)$ we run the binary search $\D{i+1}{2j-1}(v)$ that aims at maximizing the function $\left(w\mapsto \G{i+1}{2j-1}(w;v)\right)$. At iteration $N$ of $\D{i+1}{2j-1}(v)$, we have
\[
\abs{\nabla \G{i+1}{2j-1}(w_N;v)}=\abs{\nabla \H{i+2}{4j-3}(w_N)-\nabla \H{i+2}{4j-2}(v-w_N)}.
\]
We use the following estimate for $\nabla \H{i+1}{2j-1}(v)$:
\[
\hnabla \H{i+1}{2j-1}(v) \doteq \dfrac{1}{2}\left(\nabla \H{i+2}{4j-3}(w_N)+\nabla \H{i+2}{4j-2}(v-w_N)\right).
\]
Since $w\st \in (w_N,v-w_N)$ (or $(v-w_N,w_N)$), we have that \[\abs{\hnabla \H{i+1}{2j-1}(v)-\nabla \H{i+1}{2j-1}(v)} \leq \dfrac{1}{2}\abs{\nabla \G{i+1}{2j-1}(w_N;v)}.\]
Consequently we can say that with high probability,
\[
\nabla \Gij(v;u) \in \left[\hnabla \Gij(v;u)-\alpha, \hnabla \Gij(v;u)+\alpha\right]
\]
where\[\alpha=\dfrac{1}{2}\left(\abs{\nabla \G{i+1}{2j-1}(w_N;v)}+\abs{\nabla \G{i+1}{2j}(v-w_N;v)}\right).\]
In order to be sure that the algorithm does not make an error on the sign of $\nabla \Gij(v;u)$ (as in Section~\ref{sec:k2}) we have to run the binary searches $\D{i+1}{2j-1}(v)$ and $\D{i+1}{2j}(u-v)$ until $0 \notin \left[\hnabla \Gij(v;u)-\alpha, \hnabla \Gij(v;u)+\alpha\right]$ which is the case as soon as $\alpha < \abs{\nabla \Gij(v;u)}$. Therefore we decide to stop the binary $\D{i+1}{2j-1}(v)$ when $\abs{\nabla \G{i+1}{2j-1}(w_N;v)} < \abs{\nabla \Gij(v;u)}$ and to stop the binary $\D{i+1}{2j}(u-v)$ when $\abs{\nabla \G{i+1}{2j}(v-w_N;v)} < \abs{\nabla \Gij(v;u)}$.

This leads to the following lemma:
\begin{lemma}
\label{lemma:gradient_order}
During the binary search $\D{i+1}{2j-1}(v)$ we have, for all point $w$ tested by this binary search,
\[
\abs{\nabla \G{i+1}{2j-1}(w;v)} \geq \abs{\nabla \Gij(v)}.
\]
And during the binary search $\D{i+1}{2j}(v)$ we have, for all point $w$ tested by this binary search,
\[
\abs{\nabla \G{i+1}{2j}(v-w;v)} \geq \abs{\nabla \Gij(v)}.
\]
\end{lemma}

\subsection{Analysis of the Algorithm}

The rates of convergence of Theorem~\ref{TH:K4Beta} are obtained by introducing a decomposition of the overall  regret incurred by each binary search recursively run by the algorithm.
\begin{definition}
We define $\Rij(v)$ the regret induced by the binary search $\Dij(v)$ as the regret suffered when optimizing the function $\left(w \mapsto \Gij(w;v)\right)$.
\end{definition}
This notion of subregret is crucial for our induction since the regret of the algorithm after $T$ samples satisfies $R(T)=\R{0}{0}(1)/T$.

The main ingredient of the proof of Theorem~\ref{TH:K4Beta} is the following Proposition which gives a bound on the subregret $\Rij(v)$ depending on the subregrets below it.

\begin{proposition}
\label{prop:subregret2}
If $\Dij(v)$ is a node at distance $p$ from the bottom of the binary tree we have:
\[
\Rij(v) \leq \sum_{r=1}^{\rmax}8\log(2T/\delta) \dfrac{\abs*{\gij(w_r;v)}}{\abs*{\nabla \gij(w_r;v)}^2} \log(T)^p + \R{i+1}{2j-1}(w_r)+\R{i+1}{2j}(v-w_r).
\]
Where $\rmax$ the number of different samples of $\Dij(v)$ and $\gij(\point;v)\doteq\Gij(\point;v)-\Gij({z\st}^{(i)}_j(v);v)$. We have also $\rmax \leq \log_2(T)$.
\end{proposition}

This proposition is a direct consequence of the following two lemmas that are proven in Appendix~\ref{Appendix:Tech}: Lemma~\ref{lemma:subregret} gives an expression to compute the subregret $\Rij(v)$ and Lemma~\ref{lemma:gradient_samples} gives a bound on the number of samples  needed to compute $\nabla \Gij(w;v)$ at a given precision.

\begin{lemma} \label{lemma:subregret}
The subregret $\Rij(v)$ verifies
\begin{align*}
\Rij(v)&=\sum_{t=1}^{\Tij(v)} \left( \abs*{\Gij(\zij(t);v)-\Gij({z\st}^{(i)}_j(v);v))}\right) \\ &\qquad +\sum_{z \in \{\zij(t), t=1, \dots, \Tij(v) \}} \R{i+1}{2j-1}(z)+\R{i+1}{2j}(v-z)
\end{align*}
where ${z\st}^{(i)}_j(v)$ is the point where $\Gij(\cdot\,;v)$ reaches its maximum and where the successive points tested by the binary search $\Dij(v)$ are the (not necessarily distinct) $\zij(t)$.
\end{lemma}

\begin{lemma}
\label{lemma:gradient_samples}
A point $w$ tested by the binary search $\Dij(v)$ has to be sampled at most a number of times equal to
\[
8\log(2T/\delta)\dfrac{\log(T)^{p}}{\abs*{\nabla \Gij(w;v)}^2},
\]
where $p$ is the distance of the node $\Dij(v)$ to the bottom of the binary tree: $p=\log_2(K)-1-i$.
\end{lemma}

Finally it now remains  to control the different ratios $\Big|\gij(w_r;v)\Big|/\Big|\nabla \gij(w_r;v)\Big|^2$, using the \LL inequality and techniques similar to the case of $K=2$. The main difference is the binary tree we construct that imbricates binary searches. The overall idea is that each layer of that tree adds a multiplicative factor of $\log(T)$.

\subsection{Proof of Theorem \ref{TH:K4Beta} with $\beta >2$}
\label{ssec:betag2}
\begin{proof}
Let us first bound a sub-regret $\Rij(v)$ for $i \neq 0$.
Proposition~\ref{prop:subregret2} gives with $p$ the distance from $\Dij(v)$ to the bottom of the tree,
\[
\begin{split}
\Rij(v) \leq \sum_{m=1}^{\logd(T)} 8\log(2T/\delta)\dfrac{\abs*{\gij(w_m;v)}}{\abs*{\nabla \gij(w_m;v)}^2} \logd(T)^p \\+ \R{i+1}{2j-1}(w_m)+\R{i+1}{2j}(v-w_m).
\end{split}
\]
For the sake of simplicity we will note $g=\gij(\point;v)$, and we will begin by bounding
\[
R=\log(T)^p\sum_{m=1}^{\logd(T)} \dfrac{\abs*{g(w_m)}}{\abs*{\nabla g(w_m)}^2}.
\]
We use the \LL inequality to obtain that $\abs{g(w_m)}\leq c \abs{\nabla g(w_m)}^{\beta}$.
This gives
\begin{align*}
R &\leq c^{2/\beta}\logd(T)^p\sum_{m=1}^{\logd(T)} \abs*{g(w_m)}^{1-2/\beta}
\end{align*}

We are now in a similar situation as in the proof of Theorem~\ref{TH:K2Beta} in the case where $\beta>2$. Using the fact that $\abs{g(w_m)} \leq L2^{-m}$, we have
\[
R \leq \dfrac{1}{1-2^{2/\beta-1}}c^{2/\beta}L^{1-2/\beta}\logd(T)^p.
\]
Let us note $C\doteq \dfrac{1}{1-2^{2/\beta-1}}c^{2/\beta}L^{1-2/\beta}$. We have $R\leq C \logd(T)^p$.

We use now Proposition~\ref{prop:subregret2} which shows that
\[
\Rij(v) \leq 8\log(2T/\delta) \cdot C\logd(T)^p + \sum_{m=1}^{\log_2(T)} \R{i+1}{2j-1}(w_m) + \sum_{m=1}^{\logd(T)}\R{i+1}{2j}(v-w_m).
\]
Let us now define the sequence $A_p=2A_{p-1}+1$ for $p\geq 1$, and $A_0=1$.
The bound we have just shown let us show by recurrence that 
\[
\Rij(v) \leq 8\log(2T/\delta) \cdot A_p C  \log(T)^p.
\]
Lemma~\ref{lemma:arith} shows that $A_p=2^{p+1}-1\leq 2^{p+1}$. Moreover for $i=0$, we have $p=\logd(K)-1$. Consequently for $i=0$, $A_p \leq K$.

With the choice of $\delta=2/T^2$ we have finally that \[R(T)=\dfrac{\R{0}{1}(1)}{T}\leq 8 \cdot K C \dfrac{\log(T)^{\log_2(K)}}{T}\lesssim \dfrac{1}{1-2^{2/\beta-1}}c^{2/\beta}L^{1-2/\beta} K \dfrac{\log(T)^{\log_2(K)}}{T} .\]

\end{proof}

\subsection{Proof of Theorem \ref{TH:K4Beta} with $\beta =2$}

\begin{proof}

Let us first bound a sub-regret $\Rij(v)$ for $i \neq 0$.
Proposition~\ref{prop:subregret2} gives with $p$ the distance from $\Dij(v)$ to the bottom of the tree,
\[
\Rij(v) \leq \sum_{m=1}^{\logd(T)} 8\log(2T/\delta)\dfrac{\abs*{\gij(w_m;v)}}{\abs*{\nabla \gij(w_m;v)}^2} \logd(T)^p + \R{i+1}{2j-1}(w_m)+\R{i+1}{2j}(v-w_m).
\]
For the sake of simplicity we will note $g=\gij(\point;v)$ and and we will begin by bounding
\[
R=\sum_{m=1}^{\log(T)} \dfrac{\abs*{g(w_m)}}{\abs*{\nabla g(w_m)}^2}\log(T)^p.
\]
\LL inequality gives $\abs{g(w_m)}\leq c\abs{\nabla g(w_m)}^2$, leading to
\[
R \leq \sum_{m=1}^{\log(T)} c \log(T)^p \leq c \log(T)^{p+1}.
\]
\\
An immediate recurrence gives that, as in the case where $\beta>2$,
\[
\Rij(v) \leq 8 A_p c \log(2T/\delta) \log(T)^{p+1}.
\]
And finally we have, noting $g \doteq \g{0}{1}(\point;1)$ and $p=\log_2(K)-1$
\[
\R{0}{1}(1) \leq 8 A_p c \log(2T/\delta) \log(T)^{\log_d(K)}.
\]

Giving finally, with the choice $\delta=2/T^2$ and since $A_p \leq K$ for $p=\logd(K)-1$,
\[
R(T)=8 A_p c\log(2T/\delta)\dfrac{R}{T} \leq 24c K \dfrac{\log(T)^{\logd(K)+1}}{T}.
\]
\end{proof}

\subsection{Proof of Theorem \ref{TH:K4Beta} with $\beta <2$}

\begin{proof}
Let us first bound a sub-regret $\Rij(v)$ for $i \neq 0$.
Proposition~\ref{prop:subregret2} gives with $p$ the distance from $\Dij(v)$ to the bottom of the tree,
\[
\Rij(v) \leq \sum_{m=1}^{\log(T)} 8\log(2T/\delta)\dfrac{\abs*{\gij(w_m;v)}}{\abs*{\nabla \gij(w_m;v)}^2} \log(T)^p + \R{i+1}{2j-1}(w_m)+\R{i+1}{2j}(v-w_m).
\]
For the sake of simplicity we will note $g=\gij(\point;v)$ and  we will begin by bounding
\[
R=\sum_{m=1}^{\logd(T)} \dfrac{\abs*{g(w_m)}}{\abs*{\nabla g(w_m)}^2}\log(T)^p.
\]
\LL inequality gives $\abs{g(w_m)}\leq c\abs{\nabla g(w_m)}^{\beta}$, leading to
\[
R \leq \sum_{m=1}^{\logd(T)} c \abs*{\nabla g(w_m)}^{\beta-2} \logd(T)^p.
\]
We want to prove by recurrence that, with $p=\logd(K)-1-i$ and $A_p$ defined in Section~\ref{ssec:betag2}.
\[
\label{eq:rec}
\numberthis
\Rij(v) \leq 8\log(2T/\delta)cA_p\sum_{r=1}^{\rmax} \abs*{\nabla \Gij(w_r; v)}^{\beta-2} \logd(T)^p.
\]
The result is true for $p=0$ using what has be done previously.
Suppose that it holds at level $i+1$ in the tree. Then, Proposition~\ref{prop:subregret2} shows that
\begin{align*}
\Rij(v) &\leq \sum_{r=1}^{\rmax} 8\log(2T/\delta)\dfrac{\abs*{\gij(w_r;v)}}{\abs*{\nabla \Gij(w_r;v)}^2} \log(T)^p + \R{i+1}{2j-1}(w_r)+\R{i+1}{2j}(v-w_r) \\
&\leq 8\log(2T/\delta)\bigg(\logd(T)^p\sum_{r=1}^{\rmax} c \abs*{\nabla \Gij(w_r)}^{\beta-2} + \sum_{r=1}^{\rmax} cA_{p-1}\sum_{s=1}^{\smax} \abs*{\nabla \G{i+1}{2j-1}(x_s; w_r)}^{\beta-2} \logd(T)^{p-1} \\&\phantom{aaaaaaaaaaaaaaaaaaaaaaaaaaaaaaaaa}+ \sum_{r=1}^{\rmax} cA_{p-1}\sum_{s=1}^{\smax} \abs*{\nabla \G{i+1}{2j}(\tilde{x}_s; v-w_r)}^{\beta-2} \logd(T)^{p-1}\bigg).
\end{align*}
We have noted by $x_s$ and $\tilde{x}_s$ the points tested by the binary searches $\D{i+1}{2j-1}(w_r)$ and $\D{i+1}{2j}(v-w_r)$ and $\smax\leq \logd(T)$ the number of points tested by those binary searches.
We now use the fact that $\beta-2<0$ and Lemma~\ref{lemma:gradient_order} shows that $\abs*{\nabla \G{i+1}{2j-1}(x_s; w_m)} \geq \abs*{\nabla \Gij(w_r)}$, giving
\[
\Rij(v) \leq (1+2A_{p-1}) c \cdot 8 \log(2T/\delta) \sum_{r=1}^{\rmax} \abs*{\nabla \Gij(w_r;v)}^{\beta-2}\logd(T)^p,
\]
proving Equation~\eqref{eq:rec}.
And finally we have, as in the proof of Theorem~\ref{TH:K2Beta}, noting $g \doteq \g{0}{1}(\point;1)$,
\[
\R{0}{1}(1) \leq K c \cdot 8 \log(2T/\delta)\sum_{r=1}^{\rmax} \abs*{\nabla g(u_r)}^{\beta-2} \log(T)^{\log_2(K)-1}.
\]
We note now $g_r\doteq \abs*{\nabla g(u_r)}$ and we have the constraint, with $T'=\dfrac{T}{8\log(2T/\delta)\log(T)^{\logd(K)-1}}$
\[
T'=\sum_{r=1}^{\rmax}\dfrac{1}{g_r^2}.
\]
We want to maximize $R\doteq\sum_{r=1}^{\rmax} g_r^{\beta-2}$ under the above constraint.

In order to do that we introduce the following Lagrangian function:
\[
\lagr:(g_1,\dots,g_{\rmax},\lambda) \mapsto \sum_{r=1}^{\rmax} g_r^{\beta-2} + \lambda \left(T'-\sum_{r=1}^{\rmax}\dfrac{1}{g_r^2}\right).
\]
The Karush-Kuhn-Tucker theorem gives
\begin{align*}
0&=\dpart{\lagr}{g_r}(g_1,\dots,g_{\rmax},\lambda) \\
0&=(\beta-2)g_r^{\beta-3}+\lambda\left(2g_r^{-3} \right) \\
0&=(\beta-2)g_r^{\beta}+2\lambda \\
g_r&=\left(\dfrac{2\lambda}{2-\beta}\right)^{1/\beta}.
\end{align*}

The expression of $T'$ gives
\begin{align*}
T'&=\sum_{r=1}^{\rmax}g_r^{-2} \\
T'&=\sum_{r=1}^{\rmax} \left(\dfrac{2\lambda}{2-\beta}\right)^{-2/\beta}\\
\lambda^{-2/\beta}&=\dfrac{T'}{\sum_{r=1}^{\rmax} (1-\beta/2)^{2/\beta}}\\
\lambda&=T'^{-\beta/2}\rmax^{\beta/2}(1-\beta/2).
\end{align*}

We can now bound $R$:
\begin{align*}
R&\leq \sum_{r=1}^{\rmax}g_r^{\beta-2} \\
&\leq\sum_{r=1}^{\rmax}\left(\dfrac{2\lambda}{2-\beta}\right)^{1-2/\beta} \\
&\leq \rmax (1-\beta/2)^{2/\beta-1}\lambda^{1-2/\beta} \\
&\leq \rmax (1-\beta/2)^{2/\beta-1} \left(T'^{-\beta/2}\rmax^{\beta/2}(1-\beta/2)\right)^{1-2/\beta} \\
&\leq \rmax^{\beta/2} T'^{1-\beta/2}.
\end{align*}

Now we use the fact that $R(T)=\dfrac{\R{0}{1}(1)}{T}$ and $\R{0}{1}(1) \leq K c \cdot 8 \log(2T/\delta) \log(T)^{\logd(K)-1} R$.

Taking $\delta=2/T^2$, we have $\log(2T/\delta)=3\log(T)$. We have, since $\rmax \leq \log(T)$,
\begin{align*}
R(T) &\leq \dfrac{1}{T}K c \cdot 8 \log(2T/\delta) \log(T)^{\logd(K)-1} R \\
&\leq \dfrac{24 Kc}{T} \log(T)^{\logd(K)} R \\
&\leq \dfrac{24 Kc}{T} \log(T)^{\logd(K)} \rmax^{\beta/2} T'^{1-\beta/2} \\
&\leq \dfrac{24 Kc}{T} \log(T)^{\logd(K)} \log(T)^{\beta/2} \left(\dfrac{T}{24 \log(T)^{\logd(K)}}\right)^{1-\beta/2} \\
&\leq 24^{\beta/2}Kc \left(\dfrac{\log(T)^{\logd(K)+1}}{T}\right)^{\beta/2}.
\end{align*}

\end{proof}


\section{Proof of Technical and Simple Results }\label{Appendix:Tech}

\subsection{A simple lemma}
We start with a simple lemma that will be useful in the following proofs.
\begin{lemma}
\label{lemma:concav}
For all $x \in [0,1]$, we have $\abs{g(x)}\leq \abs{g'(x)}\abs{x\st-x}$.
\end{lemma}
\begin{proof}
Since $g$ is non-positive and $g(x\st)=0$, we have for all $x \in [0,1]$:
\[
\abs{g(x)}=-g(x)=g(x\st)-g(x)=\int_x^{x\st}g'(y) \dd y.
\]
Let us distinguish two cases depending on $x<x\st$ or $x>x\st$.
\begin{itemize}
\item $x<x\st$: since $g'$ is non-increasing (because $g$ is concave) we have for all $y \in [x,x\st]$, $g'(y)\leq g'(x)$ and therefore
\[
\abs{g(x)} \leq \abs{x\st-x}g'(x)=\abs{g'(x)}\abs{x\st-x}.
\]
We have indeed $g'(x)\geq 0$ because $x<x\st$.
\item $x>x\st$: similarly we have for all $y \in [x\st,x]$, $g'(y)\geq g'(x)$ and therefore
\[
\abs{g(x)}=-\int_{x\st}^{x}g'(y) \dd y \leq \abs{x\st-x}(-g'(x))=\abs{g'(x)}\abs{x\st-x}.
\]
We have indeed $g'(x)\leq 0$ because $x>x\st$ and $g'$ non-increasing.
\end{itemize}
\end{proof}

\subsection{Proof of Lemma~\ref{lemma:precision}}
\begin{proof}This lemma is just a consequence of Hoeffding inequality. Indeed, it implies that, at stage $n \in \mathds{N}$,
\[
\bP\Big\{ \Big| \frac{1}{n} \sum_{t=1}^n X_t -x \Big| \geq \sqrt{\frac{2\log(\frac{2T}{\delta})}{n}} \Big\} \leq \frac{\delta}{T}\ ,
\]
thus with probability at least $1-\delta$, $x$ belongs to $\Big[\frac{1}{N_x} \sum_{t=1}^{N_x} X_t \pm \sqrt{\frac{2\log(\frac{2T}{\delta})}{N_x}}\Big]$ and the sign of $x$ is never mistakenly determined.

On the other hand, at stage $N_x$, it holds on the same event that $\frac{1}{N_x} \sum_{t=1}^{N_x} X_t$ is $\frac{x}{2}$-close to $x$, thus 0 no longer belongs to the interval  $\Big[\frac{1}{N_x} \sum_{t=1}^{N_x} X_t \pm \sqrt{\frac{2\log(\frac{2T}{\delta})}{N_x}}\Big]$. 
\end{proof}

\subsection{A simple arithmetic lemma}
We state and prove here a simple arithmetic lemma useful in the proof of Theorem~\ref{TH:K4Beta}.
\begin{lemma}
\label{lemma:arith}
Let $(u_n)_{n\in \bN} \in \bN^{\bN}$ defined as follows: $u_0=1$ and $u_{n+1}=2u_{n}+1$.
Then \[\forall n \in \bN, \ u_n=2^{n+1}-1.\]
\end{lemma}
\begin{proof}
Let consider the sequence $v_n=u_n+1$. We have $v_0=2$ and $v_{n+1}=2v_n$. Consequently $v_n=2\cdot 2^n=2^{n+1}$.
\end{proof}

\subsection{Proof of Proposition~\ref{prop:subregret2}}

\begin{proof}
The statement of Proposition~\ref{prop:subregret2} is a restatement of Lemma~\ref{lemma:subregret} using the fact that each different point of the binary search $\Dij(v)$ is sampled a number of times equal to $8\log(2T/\delta)\dfrac{\log(T)^{p}}{\abs*{\nabla \Gij(w;v)}^2}$ thanks to Lemma~\ref{lemma:gradient_samples}. The fact that $\rmax\leq\logd(T)$ comes from the fact that running a binary search to a precision smaller than $1/LT$ does not give improved bound on the regret since the reward functions are $L$-Lipschitz continuous. Therefore the binary searches are stopped after more than $\logd(T)$ samples.
\end{proof}

\subsubsection{Proof of Lemma~\ref{lemma:subregret}}
\begin{proof}
The regret of the binary search $\Dij(v)$ is the sum for all steps $t \in [\Tij(v)]$ of the sum of two terms: the difference of the function values of $\Gij(\cdot\,;v)$ between the optimal value ${z\st}^{(i)}_j(v)$ and $\zij(t)$ and the sub-regrets $\R{i+1}{2j-1}(\zij(t))$ and $\R{i+1}{2j}(v-\zij(t))$ of the binary searches that are the children of $\Dij(v)$.
\end{proof}

\subsubsection{Proof of Lemma~\ref{lemma:gradient_samples}}

\begin{proof}
The binary search $\Dij(v)$ aims at minimizing the function $(w \mapsto \Gij(w;v))$. Let us note $w_1, \dots, w_m, \dots$ the values that are tested by this binary search. During the binary search the signs of the values of $\nabla \Gij(w_m;v)$ are needed. In order to compute them the algorithm runs sub-binary searches (unless $\Dij(v)$ is a leaf) $\D{i+1}{2j-1}(w_m)$ and $\D{i+1}{2j}(v-w_m)$.

Let us now prove the result by recurrence on the distance $p$ of $\Dij$ to the closest leaf of the tree.
\begin{itemize}
\item $p=0$: $\Dij$ is a leaf. The point $w_m$ needs to be sampled $8\log(2T/\delta)/\abs*{\nabla\gij(w_m)}^2$ (this has been shown in Section~\ref{sec:k2}).

\item $p \in \bN^*$: the point $w_m$ has to be sampled a number of times equal to the number of iterations of $\D{i+1}{2j-1}(w_m)$ and $\D{i+1}{2j}(v-w_m)$. Let us therefore compute the number of samples used by $\D{i+1}{2j-1}(w_m)$. This binary search is at distance $p-1$ of the closest leaf. Therefore by hypothesis recurrence each point $x_k$ will be sampled a number of times equal to
\[N_k=8\log(2T/\delta)\dfrac{\log(T)^{p-1}}{\abs*{\nabla\G{i+1}{2j-1}(x_k)}^2}.\]
Now Lemma~\ref{lemma:gradient_order} shows that $\abs*{\nabla\G{i+1}{2j-1}(x_k)} \geq \abs*{\nabla \Gij(w_m)}$. This gives
\[
N_k \leq 8\log(2T/\delta)\dfrac{\log(T)^{p-1}}{\abs*{\nabla \Gij(w_m)}}.
\]
The same reasoning applies for the binary search $\D{i+1}{2j}(v-w_m)$, which is run in parallel to $\D{i+1}{2j-1}(w_m)$.
Since there are at most $\log_2(T)$ different points $x_k$ that are tested during the binary search $\D{i+1}{2j-1}(w_m)$, we have a final number of iterations for $w_m$ which is
\[
8\log(2T/\delta)\dfrac{\log(T)^{p}}{\abs*{\nabla \Gij(w_m)}}.
\]
This proves the result for the step $p$.
\item Finally the recurrence is complete and the result is shown.
\end{itemize}
\end{proof}


\section{Experiments}
\label{Appendix:Expe}
In this section, we illustrate the performances of our algorithm on generated data with $K=2$ resources. We have considered different possible values for the parameter $\beta \in [1, \infty)$.

In the case where $\beta=2$ we have considered the following functions:
\[
f_1:x\mapsto \dfrac{5}{6}-\dfrac{5}{48}(2-x)^3 \quad\et\quad f_2:x\mapsto \dfrac{6655}{384}-\dfrac{5}{48}\left(\dfrac{11}{5}-x\right)^3
\]
such that $g(x)=-(x-0.4)^2$. $g$ verifies the \LL inequality with $\beta=2$ and the functions $f_1$ and $f_2$ are concave, non-decreasing and take value $0$ at $0$.

We have computed the cumulated regret of our algorithm in various settings corresponding to different values of $\beta$ and we have plotted the two references rates: the lower bound $T^{-\beta/2}$ (even if the functions considered in our examples are not those used to prove the lower bound), and the upper bound $(T/\log^2(T))^{-\beta/2}$.

Our experimental results on Figures~\ref{fig:regret_1.5}, \ref{fig:regret_1.75}, \ref{fig:regret_2} and~\ref{fig:regret_2.5} indicate that our algorithm has the correct expected behavior, as its regret is ``squeezed'' between $T^{-\beta/2}$ and $(T/\log^2(T))^{-\beta/2}$ for $\beta \leq 2$ and between $T^{-1}$ and $\log(T)/T$ for $\beta \geq 2$. Moreover, the $\log-\log$ scale also illustrates that $-\beta/2$ is indeed the correct speed of convergence for functions that satisfy the \LL inequality with respect to $\beta \in [1,2]$.

\begin{figure}[H]
 \begin{minipage}[b]{.46\linewidth}
\begin{center}
\begin{tikzpicture}[domain=0:5]
\begin{axis}
        [
        ,axis x line=bottom
  		,axis y line=center
		,ylabel near ticks,
		,xlabel near ticks,
        ,width=7cm
        ,xlabel=$T$
        ,xmin=0
        ,xmax=2000000
        ,ymin=0
        ,ymax=0.003
        ,scaled ticks=false
        ,legend style={at={(0.25,1)},anchor=north west}
        ]
        \addplot[each nth point={20}, smooth,red,line width=1.2pt,mark=triangle*, mark size=2pt] table [x index={1}, y index={4}, col sep=comma] {csv/file_1.50.csv};
        \addlegendentry{$\left(T/\log(T)^2\right)^{-\beta/2}$}
        \addplot[each nth point={20}, smooth,blue,line width=1.2pt,mark=*, mark size=2pt] table [x index={1}, y index={2}, col sep=comma] {csv/file_1.50.csv};
        \addlegendentry{$R(T)$};
        \addplot[each nth point={20},smooth, green,line width=1.2pt,mark=square*, mark size=2pt] table [x index={1}, y index={3}, col sep=comma] {csv/file_1.50.csv};
        \addlegendentry{$T^{-\beta/2}$}
    \end{axis}
\end{tikzpicture}
\subcaption{Regret as a function of $T$ \label{sub:plot}}
\end{center}
\end{minipage} \hfill
 \begin{minipage}[b]{.46\linewidth}
 \begin{center}
\begin{tikzpicture}[domain=0:5]
\begin{axis}
        [
        ,axis x line=bottom
  		,axis y line=center
		,ylabel near ticks,
		,xlabel near ticks,
        ,width=7cm
        ,xlabel=$T$
        ,scaled ticks=false
        ,legend style={at={(0.05,1)},anchor=north west}
        ]
        \addplot[smooth,red,line width=1.5pt] table [x index={5}, y index={8}, col sep=comma] {csv/file_1.50.csv};
        \addplot[smooth,blue,line width=1.5pt] table [x index={5}, y index={6}, col sep=comma] {csv/file_1.50.csv};
        \addplot[smooth,green,line width=1.5pt] table [x index={5}, y index={7}, col sep=comma] {csv/file_1.50.csv};
    \end{axis}
\end{tikzpicture}
\subcaption{Regret in $\log-\log$ scale \label{sub:log}}
\end{center}
\end{minipage}
\caption{Regret, Upper-bound and Lower bound for $\beta=1.5$\label{fig:regret_1.5}}
\end{figure}
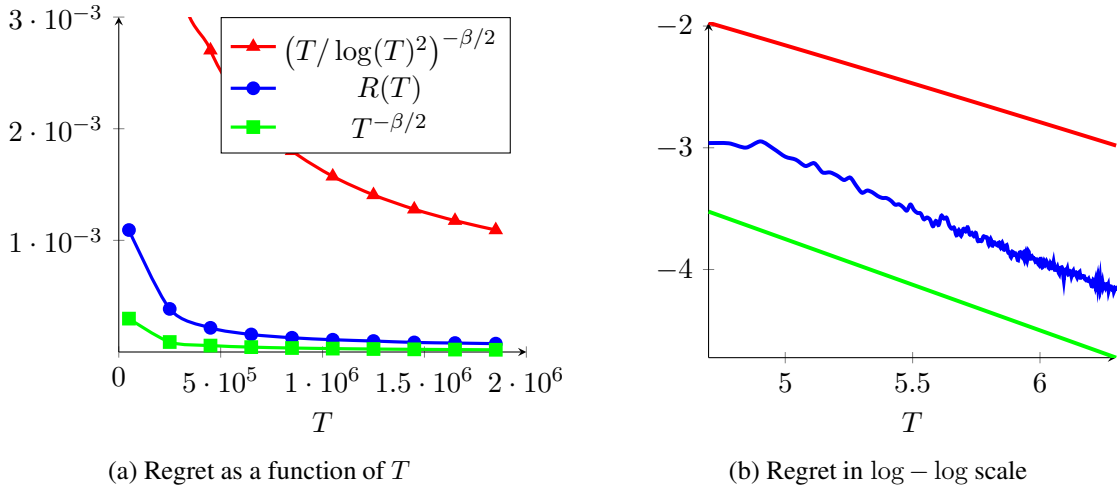

\begin{figure}[H]
 \begin{minipage}[b]{.46\linewidth}
\begin{center}
\begin{tikzpicture}[domain=0:5]
\begin{axis}
        [
        ,axis x line=bottom
  		,axis y line=center
		,ylabel near ticks,
		,xlabel near ticks,
        ,width=7cm
        ,xlabel=$T$
        ,xmin=0
        ,xmax=2000000
        ,ymin=0
        ,ymax=0.003
        ,scaled ticks=false
        ,legend style={at={(0.25,1)},anchor=north west}
        ]
        \addplot[each nth point={20}, smooth,red,line width=1.2pt,mark=triangle*, mark size=2pt] table [x index={1}, y index={4}, col sep=comma] {csv/file_1.75.csv};
        \addlegendentry{$\left(T/\log(T)^2\right)^{-\beta/2}$}
        \addplot[each nth point={20}, smooth,blue,line width=1.2pt,mark=*, mark size=2pt] table [x index={1}, y index={2}, col sep=comma] {csv/file_1.75.csv};
        \addlegendentry{$R(T)$};
        \addplot[each nth point={20},smooth, green,line width=1.2pt,mark=square*, mark size=2pt] table [x index={1}, y index={3}, col sep=comma] {csv/file_1.75.csv};
        \addlegendentry{$T^{-\beta/2}$}
    \end{axis}
\end{tikzpicture}
\subcaption{Regret as a function of $T$}
\end{center}
\end{minipage} \hfill
 \begin{minipage}[b]{.46\linewidth}
 \begin{center}
\begin{tikzpicture}[domain=0:5]
\begin{axis}
        [
        ,axis x line=bottom
  		,axis y line=center
		,ylabel near ticks,
		,xlabel near ticks,
        ,width=7cm
        ,xlabel=$T$
        ,scaled ticks=false
        ,legend style={at={(0.05,1)},anchor=north west}
        ]
        \addplot[smooth,red,line width=1.5pt] table [x index={5}, y index={8}, col sep=comma] {csv/file_1.75.csv};
        \addplot[smooth,blue,line width=1.5pt] table [x index={5}, y index={6}, col sep=comma] {csv/file_1.75.csv};
        \addplot[smooth,green,line width=1.5pt] table [x index={5}, y index={7}, col sep=comma] {csv/file_1.75.csv};
    \end{axis}
\end{tikzpicture}
\subcaption{Regret in $\log-\log$ scale}
\end{center}
\end{minipage}
\caption{Regret, Upper-bound and Lower bound for $\beta=1.75$\label{fig:regret_1.75}}
\end{figure}
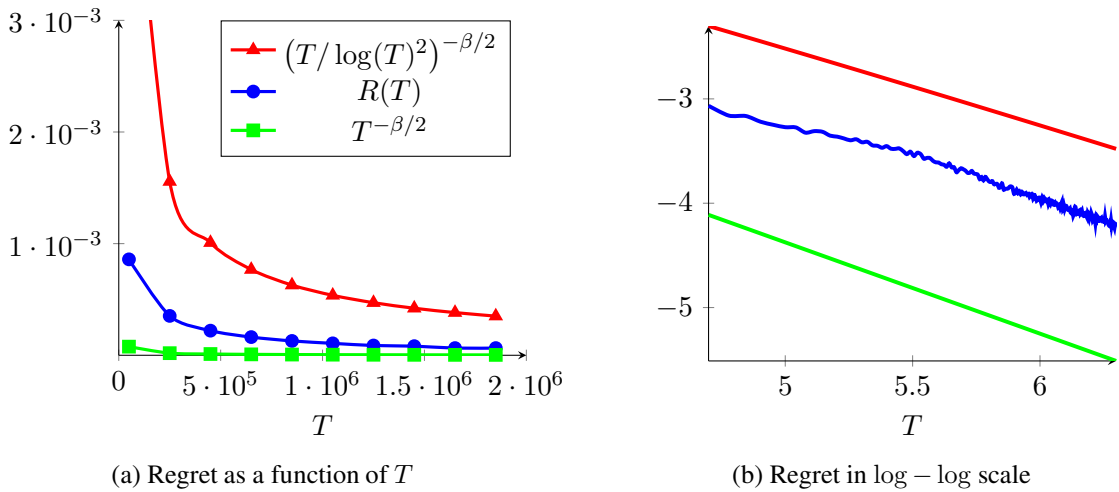

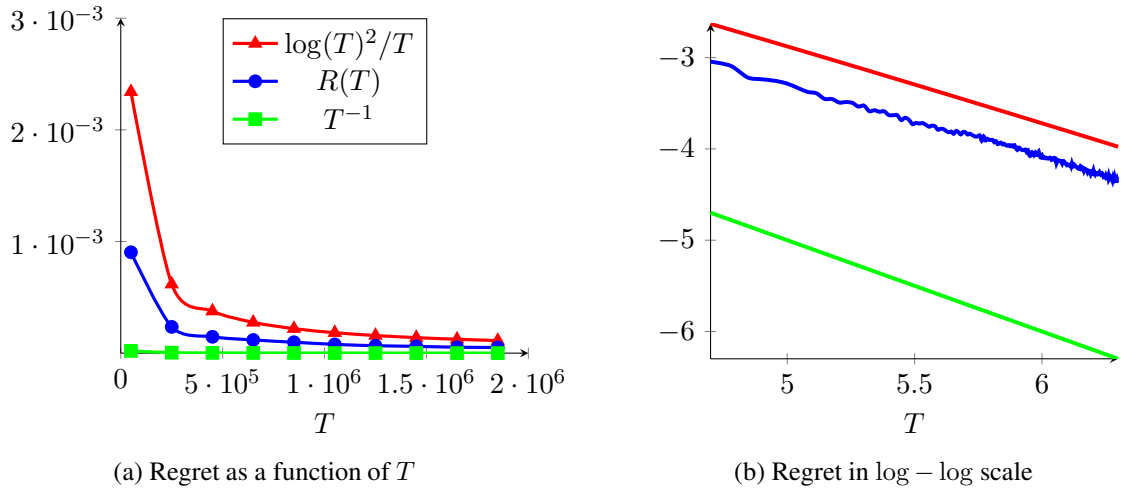
\begin{figure}[H]
 \begin{minipage}[b]{.46\linewidth}
\begin{center}
\begin{tikzpicture}[domain=0:5]
\begin{axis}
        [
        ,axis x line=bottom
  		,axis y line=center
		,ylabel near ticks,
		,xlabel near ticks,
        ,width=7cm
        ,xlabel=$T$
        ,xmin=0
        ,xmax=2000000
        ,ymin=0
        ,ymax=0.003
        ,scaled ticks=false
        ,legend style={at={(0.25,1)},anchor=north west}
        ]
        \addplot[each nth point={20}, smooth,red,line width=1.2pt,mark=triangle*, mark size=2pt] table [x index={1}, y index={4}, col sep=comma] {csv/file_2.00.csv};
        \addlegendentry{$\log(T)^2/T$}
        \addplot[each nth point={20}, smooth,blue,line width=1.2pt,mark=*, mark size=2pt] table [x index={1}, y index={2}, col sep=comma] {csv/file_2.00.csv};
        \addlegendentry{$R(T)$};
        \addplot[each nth point={20},smooth, green,line width=1.2pt,mark=square*, mark size=2pt] table [x index={1}, y index={3}, col sep=comma] {csv/file_2.00.csv};
        \addlegendentry{$T^{-1}$}
    \end{axis}
\end{tikzpicture}
\subcaption{Regret as a function of $T$}
\end{center}
\end{minipage} \hfill
 \begin{minipage}[b]{.46\linewidth}
 \begin{center}
\begin{tikzpicture}[domain=0:5]
\begin{axis}
        [
        ,axis x line=bottom
  		,axis y line=center
		,ylabel near ticks,
		,xlabel near ticks,
        ,width=7cm
        ,xlabel=$T$
        ,scaled ticks=false
        ,legend style={at={(0.05,1)},anchor=north west}
        ]
        \addplot[smooth,red,line width=1.5pt] table [x index={5}, y index={8}, col sep=comma] {csv/file_2.00.csv};
        \addplot[smooth,blue,line width=1.5pt] table [x index={5}, y index={6}, col sep=comma] {csv/file_2.00.csv};
        \addplot[smooth,green,line width=1.5pt] table [x index={5}, y index={7}, col sep=comma] {csv/file_2.00.csv};
    \end{axis}
\end{tikzpicture}
\subcaption{Regret in $\log-\log$ scale}
\end{center}
\end{minipage}
\caption{Regret, Upper-bound and Lower bound for $\beta=2$\label{fig:regret_2}}
\end{figure}

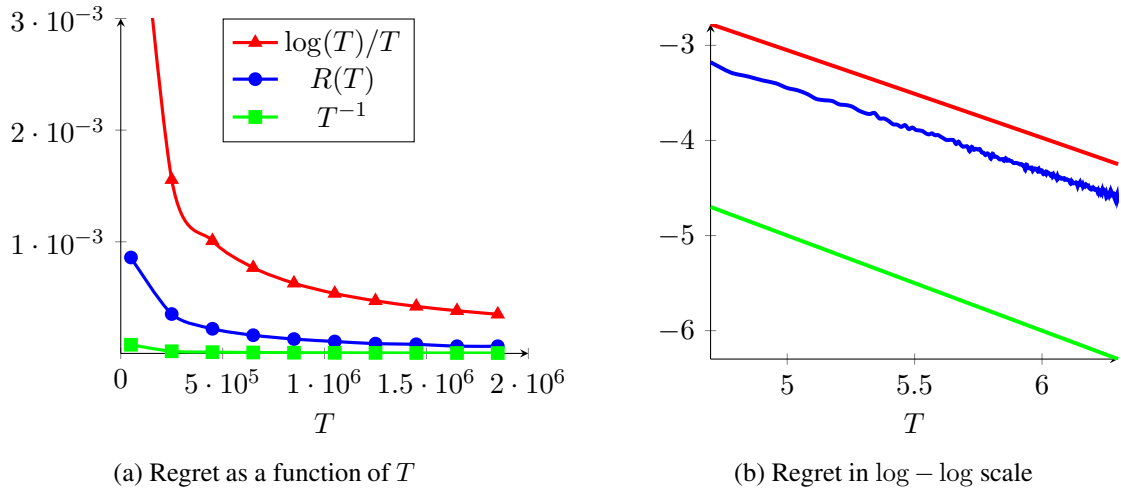
\begin{figure}[H]
 \begin{minipage}[b]{.46\linewidth}
\begin{center}
\begin{tikzpicture}[domain=0:5]
\begin{axis}
        [
        ,axis x line=bottom
  		,axis y line=center
		,ylabel near ticks,
		,xlabel near ticks,
        ,width=7cm
        ,xlabel=$T$
        ,xmin=0
        ,xmax=2000000
        ,ymin=0
        ,ymax=0.003
        ,scaled ticks=false
        ,legend style={at={(0.25,1)},anchor=north west}
        ]
        \addplot[each nth point={20}, smooth,red,line width=1.2pt,mark=triangle*, mark size=2pt] table [x index={1}, y index={4}, col sep=comma] {csv/file_1.75.csv};
        \addlegendentry{$\log(T)/T$}
        \addplot[each nth point={20}, smooth,blue,line width=1.2pt,mark=*, mark size=2pt] table [x index={1}, y index={2}, col sep=comma] {csv/file_1.75.csv};
        \addlegendentry{$R(T)$};
        \addplot[each nth point={20},smooth, green,line width=1.2pt,mark=square*, mark size=2pt] table [x index={1}, y index={3}, col sep=comma] {csv/file_1.75.csv};
        \addlegendentry{$T^{-1}$}
    \end{axis}
\end{tikzpicture}
\subcaption{Regret as a function of $T$}
\end{center}
\end{minipage} \hfill
 \begin{minipage}[b]{.46\linewidth}
 \begin{center}
\begin{tikzpicture}[domain=0:5]
\begin{axis}
        [
        ,axis x line=bottom
  		,axis y line=center
		,ylabel near ticks,
		,xlabel near ticks,
        ,width=7cm
        ,xlabel=$T$
        ,scaled ticks=false
        ,legend style={at={(0.05,1)},anchor=north west}
        ]
        \addplot[smooth,red,line width=1.5pt] table [x index={5}, y index={8}, col sep=comma] {csv/file_2.50.csv};
        \addplot[smooth,blue,line width=1.5pt] table [x index={5}, y index={6}, col sep=comma] {csv/file_2.50.csv};
        \addplot[smooth,green,line width=1.5pt] table [x index={5}, y index={7}, col sep=comma] {csv/file_2.50.csv};
    \end{axis}
\end{tikzpicture}
\subcaption{Regret in $\log-\log$ scale}
\end{center}
\end{minipage}
\caption{Regret, Upper-bound and Lower bound for $\beta=2.5$\label{fig:regret_2.5}}
\end{figure}

We plot in Figure~\ref{fig:comparison} the regret curves obtained for different values of the parameter $\beta$. This validates the fact that the convergence rates increase with the value of $\beta$ as proved theoretically.

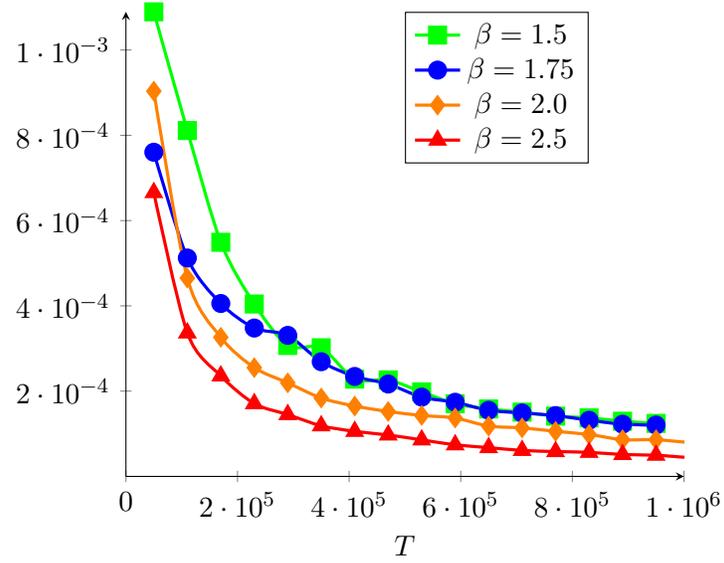
\begin{figure}[H]
\begin{center}
\begin{tikzpicture}[domain=0:5]
\begin{axis}
        [
        ,axis x line=bottom
  		,axis y line=center
		,ylabel near ticks,
		,xlabel near ticks,
        ,width=9cm
        ,xlabel=$T$
        ,xmin=0
        ,xmax=1000000
        ,ymin=0
        ,scaled ticks=false
        ,legend style={at={(.5,1)},anchor=north west}
        ]
        \addplot[each nth point={6},smooth, green,line width=1.2pt,mark=square*, mark size=3pt] table [x index={1}, y index={2}, col sep=comma] {csv/file_1.50_50.csv};
        \addlegendentry{$\beta=1.5$}
        \addplot[each nth point={6}, smooth,blue,line width=1.2pt,mark=*, mark size=3pt] table [x index={1}, y index={2}, col sep=comma] {csv/file_1.75_50.csv};
        \addlegendentry{$\beta=1.75$};
        \addplot[each nth point={6}, smooth,orange,line width=1.2pt,mark=diamond*, mark size=3pt] table [x index={1}, y index={2}, col sep=comma] {csv/file_2.00_50.csv};
        \addlegendentry{$\beta=2.0$}
        \addplot[each nth point={6}, smooth,red,line width=1.2pt,mark=triangle*, mark size=3pt] table [x index={1}, y index={2}, col sep=comma] {csv/file_2.50_50.csv};
        \addlegendentry{$\beta=2.5$}
    \end{axis}
\end{tikzpicture}
\end{center}
\caption{Regret as a Function of $T$ for different values of $\beta$ \label{fig:comparison}}
\end{figure}

\end{document}